\let\ORIlabel\label
\let\ORIrefstepcounter\refstepcounter
\AddToHook{package/hyperref/before}{
   \let\label\ORIlabel 
   \let\refstepcounter\ORIrefstepcounter}
\documentclass{siamart190516}

\usepackage{graphicx} 
\usepackage[T1]{fontenc}
\usepackage[english]{babel}
\usepackage{newlfont}

\usepackage{booktabs}

\usepackage{color}
\usepackage{amssymb}
\usepackage{amsmath}

\usepackage{cancel}
\usepackage[utf8]{inputenc}
\usepackage{algorithm}
\usepackage{algorithmic}
\usepackage{url,ulem,multirow}

\definecolor{brightpink}{rgb}{1.0, 0.0, 0.5}

\newcommand{\name}{\text{BCD}}
\newcommand{\ename}{\text{eBCD}}
\newcommand{\revise}[1]{{{\color{black} #1}}}

\usepackage{caption}
\usepackage{subcaption}
\usepackage{verbatim}

\usepackage{tikz}
\usepackage{tikz-3dplot}
\usepackage{pgfplots}
\pgfplotsset{compat=newest}

\DeclareMathOperator{\argmin}{argmin}
\DeclareMathOperator{\rank}{rank}
\DeclareMathOperator{\sgn}{sign}

\DeclareMathOperator{\nnz}{nnz} 

\title{An extrapolated and provably convergent algorithm for nonlinear matrix decomposition with the ReLU function\thanks{G.S. and N.G. acknowledge  the support by the European Union (ERC consolidator, eLinoR, no 101085607). 
The work of M.P.\ and G.S.\ was partially supported by INdAM-GNCS through Progetti di Ricerca.  The research of M.P.\  was  partially granted by the Italian Ministry of University and Research (MUR) through the PRIN 2022 ``MOLE: Manifold constrained Optimization and LEarning'',  code: 2022ZK5ME7 MUR D.D. financing decree n. 20428 of November 6th, 2024 (CUP B53C24006410006).}}
\author{Nicolas Gillis\thanks{University of Mons, Rue de Houdain 9, 7000 Mons, Belgium. Email: {\tt nicolas.gillis@umons.ac.be}}
\and Margherita Porcelli\thanks{Dipartimento di Ingegneria Industriale, Universit\`a degli Studi di Firenze, Viale Morgagni 40/44, 50134 Firenze, Italia. ISTI--CNR, Via Moruzzi 1, Pisa, Italia. Member of the INdAM Research Group GNCS. Email: {\tt margherita.porcelli@unifi.it}.} \and 
 Giovanni Seraghiti\thanks{Corresponding author. University of Mons, Rue de Houdain 9, 7000 Mons, Belgium, and 
 Dipartimento di Ingegneria Industriale, Universit\`a degli Studi di Firenze, Viale Morgagni 40/44, 50134 Firenze, Italia. Member of the INdAM Research Group GNCS. Email: {\tt giovanni.seraghiti@umons.ac.be}. 
 }
 }

\date{December 2024}

\begin{document}

\maketitle

\begin{abstract}
\revise{ReLU matrix decomposition (RMD)} is the following problem: given a sparse, nonnegative matrix $X$ and a factorization rank $r$, identify a rank-$r$ matrix $\Theta$ such that $X\approx \max(0,\Theta)$. \revise{RMD is a particular instance of nonlinear matrix decomposition (NMD) that} finds application in data compression, matrix completion with entries missing not at random, and manifold learning. 
The standard \revise{RMD} model minimizes the least squares error, that is, \mbox{$\|X - \max(0,\Theta)\|_F^2$}. The corresponding optimization problem, \revise{least-squares RMD (LS-RMD)}, is nondifferentiable and highly nonconvex. This motivated Saul  to propose an alternative model, \revise{dubbed Latent-RMD}, where a latent variable $Z$ is introduced and satisfies \mbox{$\max(0,Z)=X$} while minimizing \mbox{$\|Z - \Theta\|_F^2$}   (``A nonlinear matrix decomposition for mining the zeros of sparse data'', SIAM J.\ Math.\ Data Sci., 2022).  
Our first contribution is to show that the two formulations may yield different low-rank solutions $\Theta$.
\revise{We then consider a reparametrization of the Latent-RMD, called 3B-RMD, in which $\Theta$ is substituted by a low-rank product $WH$, where $W$ has $r$ columns and $H$ has $r$ rows.} 
Our second contribution is to prove the convergence of a block coordinate descent (BCD) \revise{approach} applied to \revise{3B-RMD.}
Our third contribution is a novel extrapolated variant of \revise{\name}, dubbed \revise{\ename}, which we prove is also convergent under mild assumptions. We illustrate the significant acceleration effect of \revise{\ename} compared to  \revise{\name}, and also show that \revise{\ename}  performs well against the state of the art on synthetic and real-world data sets.   
\end{abstract} 

\begin{keywords}
Low-rank matrix approximations, nonlinear matrix decomposition, rectified linear unit, block coordinate descent, extrapolation.  
\end{keywords} 

	\begin{AMS}
15A23, 65F55, 68Q25, 90C26, 65K05. 
	\end{AMS}

\section{Introduction} 

Nonlinear matrix decomposition (NMD) can be defined as follows:  given a data matrix $X \in \mathbb{R}^{n \times m}$, a factorization rank $r$, NMD looks for a rank-$r$ matrix $\Theta \in \mathbb{R}^{m\times n}$ such that $X\approx f(\Theta)$, where $f(\Theta)$ applies the scalar function $f$  element-wise on $\Theta$, that is, $[f(\Theta)]_{i,j}  = f(\Theta_{i,j})$ for all $i,j$. 
The choice of $f$ depends on the type of data at hand. 
For example, one might choose $f(x)=\sgn(x)$ for binary data in $\{-1,1\}$~\cite{delsarte1989low}, or  \revise{$f(x)=x^2$ for nonnegative data with a multiplicative structure}, resulting in the Hadamard product decomposition which is efficient for the compression of dense images~\cite{ciaperoni2024hadamard, hyeon2021fedpara,wertz2025efficient} and for the representation of  probabilistic circuits~\cite{loconte2024subtractive}. \revise{For a more detailed description of the choices for the function $f$;  see~\cite{awari2025alternating}.}
In this work, we assume $X$ to be nonnegative and sparse, opting for $f$ to be the ReLU function, $f(\cdot)=\max(0,\cdot)$. This NMD, with $X \approx \max(0,\Theta)$, is referred to as the \revise{ReLU matrix decomposition (RMD). This model has a close connection with a two-layer neural network, using ReLU as the activation function in the output layer; see~\cite{saul2022geometrical,saul2022nonlinear} for more details.} 
Intuitively, \revise{RMD} aims at finding a matrix $\Theta$ that substitutes the zero entries of the target matrix $X$ with negative values in order to decrease its rank. Note that some degree of sparsity on the matrix $X$ is an essential condition for \revise{RMD} to be 
 meaningful in practice. In fact, if $X$ has mostly positive entries, then $\Theta$ will also have mostly positive entries, and the ReLU nonlinear activation will be active only for few entries, leading to a model that is similar to the truncated singular value decomposition (TSVD).  

Minimizing the least squares error leads to an optimization problem referred to as \revise{LS-RMD}: 
\begin{equation}
       \min_{\Theta}\lVert X - \max(0,\Theta) \rVert_F^2 \quad \mbox{such that } \quad \rank(\Theta) \leq r.
    \label{eq:ReLU_NMD}
\end{equation} 
It is interesting to note that \revise{LS-RMD} will always lead to a lower error than the TSVD. In fact, let $X_r$ be an optimal rank-$r$ approximation of $X \geq 0$ (e.g., obtained with the TSVD), then \\
 $\|X-X_r\|_F^2 \geq \|X - \max(X_r,0)\|_F^2$ since negative entries in $X_r$, if any, are replaced by zeros which better approximate the corresponding nonnegative entries in $X$.

\revise{LS-RMD} was introduced by Saul in~\cite{saul2022geometrical,saul2022nonlinear}. It is non-differentiable and non-convex, and hence difficult to tackle. Indeed, to the best of our knowledge, 
there exist only \revise{a few} algorithms that address (\ref{eq:ReLU_NMD}) directly~\cite{awari2025alternating,awaricoordinate}.
All other existing approaches rely on Saul's alternative 
formulation, dubbed \revise{Latent-RMD}, and defined as 
\begin{equation}
     \min_{Z,\Theta} \lVert Z- \Theta \rVert_F^2 \quad 
     \mbox{such that} \quad \begin{cases} \mbox{rank}(\Theta) \leq r, \\
    \max(0,Z)=X.
    \end{cases}
    \label{eq:lat_nmd}
\end{equation} 
The objective function in (\ref{eq:lat_nmd}) is differentiable, while the subproblems in the variables $Z$ and $\Theta$ can be solved up to global optimality; see Section~\ref{sec:rel_works}. 
However, for a general matrix $X$, the link between the solutions of (\ref{eq:ReLU_NMD}) and (\ref{eq:lat_nmd}) has not been explored in the literature. Our first contribution, in Section~\ref{sec:std_vs_lat}, is to \revise{better} understand the connection between the two. 

\revise{If we reparametrize $\Theta=WH$ in~\eqref{eq:lat_nmd}, 
where $W \in \mathbb{R}^{m \times r}$ and  $H \in \mathbb{R}^{r \times n}$, as suggested in~\cite{seraghiti2023accelerated}, we obtain the following equivalent problem: }
\begin{equation}
     \min_{Z,W, H} \lVert Z - WH  \rVert_F^2 \quad 
     \text{ such that } \quad \max(0,Z)=X. 
    \label{eq:3B_nmd}
\end{equation}
\revise{We denote this reformulation 3B-RMD,} where 3B stands for 3 blocks of variables. 
The objective function in (\ref{eq:3B_nmd}) is again jointly nonconvex, but it is convex with respect to each matrix variable, $W$, $H$ and $Z$, \revise{making it more friendly from an algorithmic point of view}, as the rank constraint is not explicitly involved. \revise{Problems (\ref{eq:lat_nmd}) and} (\ref{eq:3B_nmd}) can be tackled using the block coordinate descent (BCD) method, optimizing each matrix variable alternatively~\cite{seraghiti2023accelerated,wang2024momentum}.

\paragraph{Outline and contribution} 

 Section~\ref{sec:app} motivates the study of \revise{RMD} by presenting several  applications.  
 In Section~\ref{sec:std_vs_lat}, we clarify the relationship between \revise{LS-RMD} and \revise{Latent-RMD}. In particular, we provide a  matrix $X$ for which the infimum of the rank-1 \revise{Latent-RMD} is not attained, showing the ill-poseness of \revise{this model} (\revise{Example}~\ref{th:ill_posed_lat}). On the other hand, for the same matrix $X$, \revise{LS-RMD} has a non-empty set of optimal rank-1 solutions (\revise{Example}~\ref{th:equival_th}). This shows that, in general, the two models do not share the same set of optimal low-rank solutions, that is, the sets 
    \[
    \{ \Theta^* \ | \ \Theta^* \text{ optimal for LS-RMD ~\eqref{eq:ReLU_NMD}}\}, \quad \text{and} \quad 
    \{ \widehat{\Theta} \ | \ (\widehat{Z}, \widehat{\Theta}) \text{ optimal for Latent-RMD~\eqref{eq:lat_nmd}}\}, 
    \]
    do not necessarily coincide. 
 However, we show that any feasible point \revise{$(\Bar{Z},\Bar{\Theta})$} of the \revise{Latent-RMD} model satisfies the relation $$\lVert X-\max(0,\revise{\Bar{\Theta}}) \rVert_F \leq \revise{4}\lVert \revise{\Bar{Z}}-\revise{\Bar{\Theta}} \rVert_F. $$ 
    This guarantees that any approximated solution of \revise{Latent-RMD} provides an approximated solution for \revise{LS-RMD} whose residual is at most \revise{four times} that of \revise{Latent-RMD}. 

   In Section~\ref{sec:rel_works}, we first briefly review the state-of-the-art algorithms for solving \revise{Latent-RMD} and \revise{3B-RMD} (Section~\ref{subsec:rel_works}). 
   We then prove the convergence of the BCD scheme for solving \revise{3B-RMD}  (Section~\ref{subsec:conv_bcd_nmd}), demonstrating that it falls within the framework of BCD with strictly quasi-convex subproblems~\cite{grippo2000convergence}. In addition, we introduce a novel extrapolated version of \revise{\name} (\revise{\ename}) (Section~\ref{subsec:ebcd_nmd}), adapting the well-known LMaFit algorithm for matrix completion~\cite{wen2012solving}. 
   We prove the subsequence convergence of \revise{\ename}  (Theorem~\ref{th:conv_ebcd}). To the best of our knowledge, \revise{\name} and \revise{\ename} are the first algorithms for solving \revise{3B-RMD} with convergence guarantees. 
    
    In Section~\ref{sec:num_res}, we provide extensive numerical experiments comparing our new approach, \revise{\ename}, with the state of the art. 
    We test the algorithms on various applications: matrix completion with ReLU sampling, Euclidean distance matrix completion (EDMC), compression of sparse data, and low-dimensional embedding. In particular, our experiments show that \revise{\ename} consistently accelerates \revise{\name} and performs favorably \revise{compared to} the state of the art, while having stronger theoretical guarantees. 

\section{Applications of RMD}
\label{sec:app}

In this section, we present some applications of \revise{RMD}, illustrating how nonlinear models can be used in a variety of contexts.

\subsection{Compression of sparse data}

\revise{The first application} is the compression of sparse data. 
Let $X \in \mathbb{R}^{m \times n}$ be a sparse and nonnegative matrix with $\nnz(X)$ non-zero entries, and let $W \in \mathbb{R}^{m \times r}$ and $H \in \mathbb{R}^{r \times n}$ be such that $X \approx\max(0,WH)$. 
As long as $r(m+n) < \nnz(X)$, storing the factors $W$ and $H$ requires less memory than storing $X$. 
\revise{RMD} typically outperforms standard linear compression techniques, \revise{including} the TSVD on sparse data, such as sparse images or dictionaries~\cite{awaricoordinate,saul2022geometrical,saul2022nonlinear,seraghiti2023accelerated, wang2024momentum}. 
An interesting example is the identity matrix of any dimension, which can be exactly reconstructed using a rank-3 \revise{RMD}~\cite{delsarte1989low, saul2022nonlinear}.

\subsection{Matrix completion with ReLU sampling}

The class of matrix completion problems with entries missing not at random (MNR) is characterized by the probability of an entry being missing \revise{that} depends on the matrix itself. For example, when completing a survey, people are likely to share nonsensitive information such as their name or surname, while they might be more hesitant to disclose their phone number. Therefore, we might expect the missing entries to be concentrated in some areas of the matrix more than others. Even though MNR matrix completion is still a relatively unexplored area, dropping the assumption that missing entries are independent of the matrix is a challenging problem that attracts more and more \revise{researchers} as the increasing number of related works testifies~\cite{bhattacharya2022matrix,bi2021absence,ganti2015matrix,liu2024symmetric,naik2022truncated}. Among them, Liu et al.~\cite{liu2024symmetric} investigate the matrix completion problem when only positive values are observed, namely matrix completion with ReLU sampling. This instance of matrix completion falls within the MNR problems since the missing entries depend on the sign of the target matrix. Assume that $X$ is a nonnegative and sparse matrix and that there exists a rank-$r$ matrix $\Theta$ such that  $X=\max(0,\Theta)$. In this case, finding an \revise{RMD} is equivalent to recovering the negative values observing the positive ones; thus, it is equivalent to matrix completion with ReLU sampling. This implies that we can use the algorithms for \revise{RMD} to address this matrix completion problem, as illustrated in~\cite{liu2024symmetric}.

Inspired by the connection to ReLU sampling matrix completion, we introduce an equivalent formulation of \revise{3B-RMD} that we will use throughout the paper. Let $ \Omega=\{(i,j) \ | \ X_{ij}>0 \}$ be the set of positive entries of $X$, and let $P_{\Omega}(X)$ be equal to $X$ in $\Omega$ and zero elsewhere. 
We rewrite \revise{3B-RMD} in (\ref{eq:3B_nmd}) as
\begin{equation}
     \min_{Z,W,H} \frac{1}{2} \lVert Z-WH \rVert_F^2  \quad 
     \mbox{such that} \quad 
    P_{\Omega}(Z)=P_{\Omega}(X),\quad P_{\Omega^C}(Z) \leq 0, 
    \label{eq:lat_3B_matr_com}
\end{equation}
where $\Omega^C$ is the complement of $\Omega$ and a reformulation of the constraint $\max(0,Z)=X$ is introduced. To our knowledge, the problem (\ref{eq:lat_3B_matr_com}), without the constraint $P_{\Omega^C}(Z) \leq 0$, was first addressed in~\cite{wen2012solving} in the context of matrix completion. The novelties in our model are the additional information that the missing entries have negative signs, which translates into the constraints $P_{\Omega^C}(Z) \leq 0$, and the explicit dependence of the known index set $\Omega$ \revise{on} the target matrix $X$. Since the formulations (\ref{eq:3B_nmd}) and (\ref{eq:lat_3B_matr_com}) are equivalent, we will refer to both as \revise{3B-RMD}.

\subsection{Euclidean distance matrix completion (EDMC)}

We provide one practical example of MNR matrix completion problem that can be solved using \revise{RMD}. \revise{Specifically, we address} the recovery of a low-rank matrix for which only the smallest or the largest entries are observed. Let $\Theta$ be a low-rank matrix, and assume that we observe all the entries in $\Theta$ smaller than a known threshold $d$. We want to recover the missing entries, which are the ones larger than $d$. We can model this problem using a rank-1 modified \revise{RMD}: $X=\max(0,d ee^T - \Theta)$, where $e$ denotes the vector of all ones of appropriate dimension. The matrix $X$ is nonzero exactly in the entries where $\Theta$ is smaller than $d$. Of course, an analogous model can be derived when only the largest entries are observed. 
The \revise{3B-RMD} formulation in (\ref{eq:lat_3B_matr_com}) can be easily adapted to include the rank-1 modification: 
\begin{equation}
     \min_{Z,W,H} \frac{1}{2} \lVert d e e^T - WH - Z \rVert_F^2  \quad 
     \mbox{such that} \quad 
    P_{\Omega}(Z)=P_{\Omega}(X),\quad P_{\Omega^C}(Z) \leq 0. 
    \label{eq:lat_3B_matr_com_rank1}
\end{equation}
The same idea applies also to the \revise{Latent-RMD} and all the algorithms can be adjusted to take into account the additional rank-one term. This formulation is particularly interesting when the matrix $\Theta$ contains the squared distances between a collection of points $p_1,\dots,p_n \in \mathbb{R}^d$. This problem is often referred to as Euclidean distance matrix completion (EDMC)~\cite{alfakih1999solving,fang2012Euclidean,krislock2012euclidean,tasissa2018exact}, and is particularly meaningful in sensor network localization problems where sensors can only communicate with nearby sensors. 
Let $P =[p_1,\dots,p_n] \in \mathbb{R}^{d\times n}$, then the corresponding Euclidean distance matrix is $\Theta_{i,j} = \|p_i - p_j\|_2^2$ for all $i,j$, so that 
 $\Theta = e\text{diag}(P^TP)^T + \text{diag}(P^TP) e^T + 2 P^T P$ and it has at most rank $r=d+2$ and thus the rank of the factorization in (\ref{eq:lat_3B_matr_com_rank1}) is known in advance. We will show in Section~\ref{sec:num_res} that this slightly modified \revise{RMD} allows us to solve EDMC in this context.

\subsection{Manifold learning} \label{sec:maniflearn}

Manifold learning aims at learning a norm-preserving and neighborhood-preserving mapping of high-dimensional inputs into a lower-dimensional space~\cite{lee2007nonlinear}.  Given a collection of points $\{z_1,\dots, z_m\}$ in $\mathbb{R}^N$, Saul~\cite{saul2022geometrical} looks for a faithful representation $\{y_1,\dots,y_m\}$  in $\mathbb{R}^r$, with $r < N$. Let $\tau \in (0,1)$, an embedding is $\tau$-faithful if
\begin{equation} 
    \max(0,\langle y_i, y_j \rangle - \tau \lVert y_i \rVert \lVert y_j \rVert) = \max(0,\langle z_i, z_j \rangle - \tau \lVert z_i \rVert \lVert z_j \rVert ) . 
    \label{eq:tau_faith}
\end{equation}
Equation (\ref{eq:tau_faith}) implies that the embedding must preserve norms (take $i=j$ to obtain \mbox{$(1-\tau)\lVert y_i  \rVert^2$} $=$ \mbox{$(1-\tau)\lVert z_i  \rVert^2$}) 
as well as the small angles between two points $(z_i, z_j)$: when $\cos(z_i,z_j) > \tau$, we need 
$\cos(z_i,z_j) = \cos(y_i,y_j)$. 
Large angles have to remain large but do not need to be preserved, that is, $\cos(z_i,z_j) \leq \tau$ only requires $\cos(y_i,y_j) \leq \tau$. The condition~\eqref{eq:tau_faith}  defines a similarity between points based on norm and angles, which differs substantially from the more common concept of pairwise distances. This makes it suitable, for example, in text analysis where the $z_i$'s are sparse vectors of word counts. 

This embedding strategy is dubbed threshold similarity matching (TSM). Saul proposes three main steps to compute such an embedding. First, for each couple of points $(z_i, z_j)$, compute the right-hand side in (\ref{eq:tau_faith}) and construct the sparse, nonnegative similarity matrix 
$$ 
X_{ij}= \max(0,\langle z_i, z_j \rangle - \tau \lVert z_i \rVert \lVert z_j \rVert ) \quad \text{ for all } i,j. 
$$
Second, compute an \revise{RMD}, $X \approx \max(0,\Theta)$, and the matrix $\max(0,\Theta)$ represents the similarities between points in the lower-dimensional space. Finally, the embedded points are recovered  from the low-rank approximation $\Theta$; see~\cite{saul2022geometrical} for more details.

\section{ LS-RMD vs Latent-RMD} 
\label{sec:std_vs_lat}

In this section, we explore the link between \revise{LS-RMD} in (\ref{eq:ReLU_NMD}) and \revise{Latent-RMD} in (\ref{eq:lat_nmd}). In particular, we aim to determine if the optimal low-rank approximations provided by the two models coincide. First, when  $X$ contains negative values, \revise{Latent-RMD} is infeasible since there exists no latent variable $Z$ such that  $X = \max(0,Z)$. In contrast, \revise{LS-RMD} still admits feasible solutions. Of course, the \revise{RMD} is more meaningful when $X$ is nonnegative, which is the case we focus on in the following. 

In the exact case, if there exists a low-rank matrix $\Theta$ such that $X=\max(0,\Theta)$, \revise{LS-RMD} and \revise{Latent-RMD} are equivalent: for any low-rank solution $\Theta^*$ of (\ref{eq:ReLU_NMD}), there exists a latent matrix $Z^*$, such that $(Z^*,\Theta^*)$ is a solution of (\ref{eq:lat_nmd}) and vice versa. Indeed, in this case, there always exists a trivial latent solution $Z^*=\Theta^*$. We show in what follows that this is not \revise{always} true \revise{if} an exact decomposition does not exist.

We demonstrate first that there exists a specific matrix $X$ such that the rank-1 \revise{Latent-RMD} is ill-posed since the infimum of \eqref{eq:lat_nmd} is not attained 
(see \revise{Example}~\ref{th:ill_posed_lat}). Moreover, by showing that for the same matrix $X$, \revise{LS-RMD} admits a non-empty set of optimal solutions (\revise{Example}~\ref{th:equival_th}), we establish that the two formulations do not necessarily yield the same optimal approximations. 
Note that we could not establish the general well-poseness of \revise{LS-RMD}, which is a direction of future research. 
\begin{example}   \label{th:ill_posed_lat}
Let $0 < \epsilon < 1 / \sqrt{2}$ be a fixed parameter, $r=1$, and     
\begin{equation}   
X=\begin{pmatrix}
    1 & 0\\
    \epsilon & 1
\end{pmatrix}. 
\label{eq:counter_matrix}
\end{equation} 
Then the infimum of \revise{Latent-RMD} in (\ref{eq:lat_nmd}) is  not attained.
\end{example} 

\begin{proof}
 Let us prove the result by contradiction. Assume there exists a solution to (\ref{eq:lat_nmd}), denoted $(\widehat{Z},\widehat{\Theta})$. 
 \revise{It is straightforward to notice that, independently of} $\widehat{Z}$, the optimal $\widehat{\Theta}$ must be a best rank-1 approximation of $\widehat{Z}$.
Let us denote the $i$-th singular value of $\widehat{Z}$ as $\sigma_i(Z)$, and $\widehat{Z}_1$ be a best rank-one approximation of $\widehat{Z}$. 
By the Eckart-Young theorem, $\|\widehat{Z} - \widehat{Z}_1\|_F^2 = \sigma_2(Z)^2$ since $\widehat{Z}$ is a 2-by-2 matrix.  
Hence $\widehat{Z}$ must belong to 
\begin{equation}
    \argmin_Z \sigma_2(Z)^2 \quad \mbox{such that } \quad \max(0,Z)=X.
   \label{eq:svd_latent_2} 
\end{equation} 
Equivalently, we want to find $b \leq 0$ such that the matrix $Z=\begin{pmatrix}
    1 & b\\
    \epsilon & 1
\end{pmatrix}$ has the lowest possible second singular value. Since $Z$ is a $2\times 2$ matrix, we can compute the eigenvalues of $ZZ^T$ by solving det$(ZZ^T-\sigma I)=0$ and get an explicit expression for the second singular value of $Z$. \revise{Solving the equation allows us to reformulate~\eqref{eq:svd_latent_2} as }
\begin{equation}
     \argmin_{b<0} \; \ell(b) :=\frac{2+b^2+\epsilon^2-\sqrt{(b^2-\epsilon^2)^2+4(b+\epsilon)^2}}{2}.
    \label{eq:equival_d}
\end{equation} 
One can verify that $\lim_{b \rightarrow -\infty} \ell(b) =  \epsilon^2$, while $\ell(b) > \epsilon^2$ for every $\epsilon \leq 1 / \sqrt{2}$ and $b < 0$. This means the infimum of $\ell(b)$ is not attained by any finite value of $b$, and hence $\widehat{Z}$ does not exist. 
\end{proof}

We now demonstrate that rank-1 \revise{LS-RMD} admits an optimal solution for $X$ as in \eqref{eq:counter_matrix}.

\begin{example}
Let $X$ as in \eqref{eq:counter_matrix} and $r=1$, then for every $0<\epsilon<\frac{1}{\sqrt{2}}$, the infimum of \revise{LS-RMD} is attained by all the matrices of the form
\begin{equation}  
\Theta=\begin{pmatrix}
    1 & -v\\
    -\frac{1}{v} & 1
\end{pmatrix}, \quad  v \in \mathbb{R_+},
\label{eq:optimal_matrix}
\end{equation}
with optimal value equal to $\epsilon^2.$
   \label{th:equival_th}
\end{example}
\begin{proof}
Let us denote the error function of \revise{LS-RMD} as $F(\Theta)=\lVert X-\max(0,\Theta) \rVert_F^2$. 
Note that $2 \times 2$ rank-1 matrices have limited patterns that the signs of their entries must follow. For instance, if the matrix does not contain zeros, the number of positive and negative entries must be even; otherwise, the rank cannot be equal to one. 
We proceed by analyzing three different cases. 

\textit{Case 1:} The rank-1 matrix $\Theta$ contains one zero column or row. In this case, the residual $F(\Theta)$ cannot be less than 1 since the nonlinear approximation $\max(0,\Theta)$ has at least one zero in the diagonal.  

\textit{Case 2:} The rank-1 matrix $\Theta$ has no zero rows or columns, and it contains two negative entries. 
If there is a negative entry on the diagonal, then the error is larger than one. Otherwise, the positive entries on the diagonal can be set to 1 to minimize the error, which is exactly the matrix in (\ref{eq:optimal_matrix}), with error $\epsilon^2$.  

\textit{Case 3:} The rank-1 matrix $\Theta$ contains only positive entries. If $\Theta \geq 0$,  \revise{ the nonlinear activation has no effect on $\Theta$ and the ReLU-NMD corresponds to minimizing the function $\lVert X- \Theta \rVert_F^2$.} 
Let $\widehat{\sigma}_2$ be the second singular value of $X$ as in \eqref{eq:counter_matrix}. If rank($\Theta$)=1, then $\lVert X- \Theta \rVert_F^2\geq \widehat{\sigma}_2^2$ and using equation \eqref{eq:equival_d}, \revise{we have}
$\widehat{\sigma}_2^2=\frac{2+\epsilon^2-\epsilon \sqrt{\epsilon^2+4}}{2}.$ 
Simple calculations show that $\widehat{\sigma}_2^2$ is strictly larger than $\epsilon^2$ for any $0<\epsilon<\frac{1}{\sqrt{2}}$.
\end{proof}
As a consequence of \revise{Example}~\ref{th:ill_posed_lat} and \revise{Example}~\ref{th:equival_th}, we have the following corollary.
\begin{corollary}
\revise{LS-RMD} and \revise{Latent-RMD} may not share the same set of optimal low-rank solutions. 
\label{cor:not_equivalence}
\end{corollary}

We established the non equivalence between the low-rank solutions of \revise{LS-RMD} and \revise{Latent-RMD}. However, in practice, we are interested in computing accurate solutions of even though they might not be optimal. Numerous numerical results show that the decompositions obtained by solving \revise{the Latent-RMD model} achieves a small residual error also for \revise{LS-RMD }~\cite{saul2022nonlinear, seraghiti2023accelerated, wang2024momentum}. The following theorem formalizes this observation, showing that any feasible point \revise{$(\Bar{\Theta}, \Bar{Z})$} of the latent model has a residual value for in the original \revise{LS-RMD} that is upper bounded by \revise{four times} the residual of \revise{Latent-RMD}. 
\begin{theorem}
\label{th:bound_for}
Let $X \in \mathbb{R}^{n \times m}$ be a nonnegative matrix and denote as \revise{$(\Bar{\Theta}, \Bar{Z})$} a feasible point of the \revise{Latent-RMD} model in (\ref{eq:lat_nmd}); 
then 
\begin{equation}
    \lVert X-\max(0,\revise{\Bar{\Theta}}) \rVert_F^2 \leq  \revise{4} \lVert \revise{\Bar{Z}}-\revise{\Bar{\Theta}} \rVert_F^2.
    \label{eq:rel_lat_real}
\end{equation}
\end{theorem}
\begin{proof}
\revise{Using the matrix inequality $\lVert A + B \rVert_F^2 \leq 2 \lVert A \rVert_F^2 + 2 \lVert B \rVert_F^2$,} we have 
    \begin{equation}
    \begin{aligned}  
   \lVert X-\max(0,\revise{\Bar{\Theta}}) \rVert_F^2&=\lVert X-\revise{\Bar{Z}}+\revise{\Bar{Z}}\revise{+\min(0,\Bar{\Theta}) -\min(0,\Bar{\Theta}) -\max(0,\Bar{\Theta})} \rVert_F^2\\
   &\leq \revise{2} \lVert X-\revise{\Bar{Z}}+\min(0,\revise{\Bar{\Theta}}) \rVert_F^2+ \revise{2} \lVert \revise{\Bar{Z}}-\revise{\Bar{\Theta}} \rVert_F^2. \end{aligned} 
   \label{eq:first_eq}
   \end{equation}
Let us define 
$\Omega = \{ (i,j) \ | \ X_{i,j} > 0\}$
and $\widehat \Omega = \{ (i,j) \ | \ \revise{\Bar{\Theta}}_{i,j} > 0\}$, and $\Omega^C$ and $\widehat \Omega^C$ their respective complements. 
By assumption, $\revise{\Bar{Z}}$ is a feasible solution of (\ref{eq:lat_nmd}); hence  $\max(0,\revise{\Bar{Z}})=X$. Therefore, 
$$
\begin{aligned}
    \lVert \revise{\Bar{Z}}-\revise{\Bar{\Theta}} \rVert_F^2&=\sum_{(i,j) \in \Omega\cap \widehat \Omega}(X_{ij}-\revise{\Bar{\Theta}}_{ij})^2+\sum_{(i,j) \in \Omega\cap \widehat{\Omega}^C}(X_{ij}-\revise{\Bar{\Theta}}_{ij})^2 + \\
    & +\sum_{(i,j) \in\Omega^C\cap \widehat{\Omega}}(\revise{\Bar{Z}}_{ij}-\revise{\Bar{\Theta}}_{ij})^2+ \sum_{(i,j) \in\Omega^C\cap \widehat{\Omega}^C}(\revise{\Bar{Z}}_{ij}-\revise{\Bar{\Theta}}_{ij})^2.
\end{aligned}
$$

Similarly, $\revise{\Bar{Z}}_{i,j}\leq 0$ whenever $X_{i,j}=0$; \revise{thus,} 
\begin{align}
\lVert X-\revise{\Bar{Z}}+\min(0,\revise{\Bar{\Theta}}) \rVert_F^2
& =\sum_{(i,j) \in\Omega\cap\widehat{\Omega}^C}\revise{\Bar{\Theta}}_{ij}^2 
+\sum_{(i,j) \in\Omega^C\cap \widehat{\Omega}}\revise{\Bar{Z}}_{ij}^2+ \sum_{(i,j) \in\Omega^C\cap \widehat{\Omega}^C}(\revise{\Bar{Z}}_{ij}-\revise{\Bar{\Theta}}_{ij})^2 \nonumber \\
&\leq \sum_{(i,j) \in\Omega\cap \widehat{\Omega}^C}(X_{ij}-\revise{\Bar{\Theta}}_{ij})^2 +\sum_{(i,j) \in\Omega^C\cap \widehat{\Omega}}(\revise{\Bar{Z}}_{ij}-\revise{\Bar{\Theta}}_{ij})^2 + \hspace{-0.2cm} \sum_{(i,j) \in\Omega^C\cap \widehat{\Omega}^C}(\revise{\Bar{Z}}_{ij}-\revise{\Bar{\Theta}}_{ij})^2 \nonumber 
 \\
&\leq \lVert \revise{\Bar{Z}}-\revise{\Bar{\Theta}} \rVert_F^2,
\label{eq:second_eq}
\end{align}  
\revise{where} the first inequality holds because 
$$ 
\revise{\Bar{\Theta}}_{ij}^2\leq (X_{ij}-\revise{\Bar{\Theta}}_{ij})^2 \mbox{ for } (i,j) \in \Omega\cap \widehat{\Omega}^C, 
\quad \text{ and } \quad \revise{\Bar{Z}}_{ij}^2\leq (\revise{\Bar{Z}}_{ij}-\revise{\Bar{\Theta}}_{ij})^2 \mbox{ for } (i,j) \in \Omega^C\cap \widehat{\Omega}. 
$$ 
Combining (\ref{eq:first_eq}) and (\ref{eq:second_eq}), we get (\ref{eq:rel_lat_real}).
\end{proof}
A straightforward consequence of Theorem~\ref{th:bound_for} is that a solution of \revise{Latent-RMD} with a small error provides a \revise{LS-RMD} solution with a small error, even though the two models do not necessarily yield the same approximation. \revise{We specify that the same results holds if we reparametrize $\Theta$ as $WH$; thus, the same bound holds for 3B-RMD as well.} 

\section{Algorithms for RMD}
\label{sec:rel_works}

This section is dedicated to the algorithms for computing \revise{RMDs} and to the study of their convergence. In 
Section~\ref{subsec:rel_works}, we briefly review the state-of-the-art algorithms for solving \revise{Latent-RMD} and \revise{3B-RMD}. In Section~\ref{subsec:conv_bcd_nmd}, we discuss the convergence of the \revise{\name} scheme. Finally, we present our new algorithm \revise{\ename} along with its convergence in 
Section~\ref{subsec:ebcd_nmd}. 

\subsection{Previous works}
\label{subsec:rel_works}
Let us start with \revise{ the algorithm designed for} \revise{Latent-RMD} since we use \revise{them} as baselines in the numerical comparison in Section~\ref{sec:num_res}. 
To our knowledge, all existing algorithms are alternating optimization methods in which each block of variables is updated sequentially~\cite{grippo2000convergence,grippo2023introduction,patriksson1998decomposition,powell1973search,tseng1991decomposition} and lack convergence guarantees. 

The first approach is the so-called \revise{Naive algorithm}, introduced by Saul in~\cite{saul2022nonlinear}, which employs a basic alternating scheme over $Z$ and $\Theta$ with closed-form solutions for both. Similar approaches have been studied in matrix completion, see for example~\cite{balzano2010online,boumal2015low,dai2012geometric,goyens2023nonlinear}. At each iteration, firstly $\Theta$ is fixed and then the optimal $Z$ is computed  solving the subproblem
\begin{equation}
    \argmin_Z \lVert Z-\Theta \rVert_F^2 \quad \mbox{such that} \quad \max(0,Z)=X. 
    \label{eq:sub_Z}
\end{equation}
The solution of (\ref{eq:sub_Z}) is given by the projection of $\Theta$ over the set $\{Z \ : \ \max(0,Z)=X\}$, that is, 
$$
    Z_{ij}=\begin{cases}
        X_{ij} \quad & \mbox{if } X_{ij} > 0\\
        \min(0,\Theta_{ij}) & \mbox{if } X_{ij} = 0.\\        
    \end{cases}
$$ 
Once $Z$ is updated,  $\Theta$ is updated  solving
    $\argmin_{\Theta} \lVert Z-\Theta \rVert_F^2$ such that 
    $\rank(\Theta) \leq r$,  
whose optimal solution(s) can be obtained with the TSVD of $Z$. Despite its simplicity, this scheme is highly reliable in practice, with the most costly operation at each iteration being the computation of a TSVD. 

An accelerated, heuristic version of the Naive scheme was proposed in~\cite{seraghiti2023accelerated}. Specifically, the so-called \revise{Aggressive Naive (A-Naive) algorithm,} employs an extrapolation step for both $Z$ and $\Theta$ while the extrapolation parameter changes adaptively at each iteration, depending on the value of the residual. \revise{A-Naive} is an example where, despite its lack of \revise{convexity and} theoretical guarantees, extrapolation can be considerably faster \revise{the non accelerated counterpart}~\cite{ang2019accelerating,wen2012solving}.

Another algorithm for solving \revise{Latent-RMD} is the expectation-minimization \revise{(EM) algorithm}~\cite{saul2022nonlinear}. This approach is based on a Gaussian latent variable model parametrized by the low-rank matrix $\Theta$ and a parameter $\sigma$. For each entry of the original matrix, a Gaussian latent variable $Z_{ij}=\mathcal{N}(\Theta_{ij},\sigma^2)$
is sampled. The method estimates $\Theta$ and $\sigma$ by maximizing the likelihood of the data with respect to $\Theta$ and $\sigma$. The optimization is carried out using the expectation-minimization scheme that consists of two main steps: the E-step which computes posterior mean and variance of the Gaussian latent variable model, and the M-step re-estimates the parameters of the model from the posterior mean and variance.

\revise{Finally, we present the \name\ algorithm for solving the 3B-RMD model. In what follows, we use the formulation of 3B-RMD introduced in (\ref{eq:lat_3B_matr_com}) rather than the equivalent one in~\eqref{eq:3B_nmd}. The \name\ algorithm} consists of an alternating optimization procedure over three blocks of variables $(Z,W,H)$, \revise{and computes} a global optimizer for each of the subproblems sequentially. Let $A^\dagger$ denote the Moore-Penrose inverse of $A$, and let $(H^k,Z^k,W^k)$ be the current iteration. 
The $(k+1)$-th step of \revise{\name} is as follows: 
\begin{equation}
    \begin{aligned}
         Z^{k+1} =& \argmin_Z \{\lVert Z-W^kH^k \rVert_F^2 \quad \text{s.t.} \quad \max(0,Z)=X \} =P_{\Omega}(X)+P_{\Omega^C}(\min(0,W^kH^k)),\\
        W^{k+1} =&\argmin_W \lVert Z^{k+1}-W H^k \rVert_F^2=Z^{k+1}(H^k)^{\dagger}, \\
        H^{k+1} =& \argmin_H \lVert Z^{k+1}-W^{k+1} H \rVert_F^2=((W^{k+1})^T)^{\dagger}Z^{k+1}, \\      
    \end{aligned}
    \label{eq:BCD_NMD}
\end{equation} 
\revise{where $P_{\Omega}$ and $P_{\Omega}^C$ are defined as in (\ref{eq:lat_3B_matr_com}).} The solutions of the subproblems for $W \in \mathbb{R}^{m \times r}$ and $H \in \mathbb{R}^{r \times n}$ are analogous. In particular, the first is separable by rows of $W$, leading to $m$ least squares problems in $r$ variables with data vector of dimension $n$, while the subproblem for $H$ is separable by column, and equivalent to $n$ least squares problems of $r$ variables with data vector of dimension $m$. Both these subproblems do not necessarily have unique solutions. In fact, uniqueness is guaranteed only if $W^k$ and $H^k$ remain full rank throughout the iterations. Alternatively, uniqueness is ensured by adding a Tikhonov regularization term or a proximal term in each of the subproblems~\cite{attouch2010proximal,grippo2000convergence,grippo2023introduction}. However, \revise{regularization might slow down the algorithm, and} it is not necessary to prove the convergence of the \revise{\name\ scheme in this context.}

On the contrary, the subproblem in the latent variable $Z$ admits a unique global minimum, which is the projection of the product $W^kH^k$ onto the feasible, convex set 
\begin{equation}
    \{ Z \in \mathbb{R}^{m\times n} \ : \  P_{\Omega}(Z)=P_{\Omega}(X), \ P_{\Omega^C}(Z) \leq 0\}.
    \label{eq:conv_set_Z}
\end{equation} 
The overall \revise{\name} algorithm is reported in Algorithm \ref{alg:BCD_NMD} and has a computational cost of $O(r|\Omega^C|+(m+n)r^2)$. 

\begin{algorithm}[h]
 \caption{\revise{\name}: block coordinate descent for \revise{3B-RMD}}
\begin{algorithmic}[1]
        \REQUIRE $X$, $W^0$, $H^0$, maxit. 

        \ENSURE low-rank factors $W^k$ and $H^k$. 
        \medskip   
        
        \STATE Set $\Omega=\{(i,j) \ | \ X_{ij}>0 \}$
        \FOR{$k=0,1,\dots,$ maxit}
         \STATE $Z^{k+1}=P_{\Omega}(X)+P_{\Omega^C}(\min(0,W^{k}H^{k}))$
        \STATE $W^{k+1}=Z^{k+1} (H^k)^{\dagger}$
        \STATE $H^{k+1}=((W^{k+1})^T)^\dagger Z^{k+1}$
       
        \ENDFOR
\end{algorithmic}
\label{alg:BCD_NMD}
 \end{algorithm}

Similarly to the Naive scheme for \revise{Latent-RMD}, adding extrapolation might be beneficial also for \revise{\name}. 
In particular, the work~\cite{seraghiti2023accelerated} proposes to perform an extrapolation step on the latent variable $Z$ and on the low-rank product $WH$, resulting in the so-called \revise{extrapolated 3B (e3B)} algorithm. Moreover, Wang et al.~\cite{wang2024momentum,wang2025accelerated} added a Tikhonov regularization for both the variables $W$ and $H$. Their algorithm is based on the \revise{e3B} scheme but performs two additional extrapolation steps following the computation of $W$ and $H$. Even though numerical experience shows that both accelerated algorithms outperform the original \revise{\name}, convergence still remains an open problem.
 
\subsection{Convergence of \name}
\label{subsec:conv_bcd_nmd}

In this section, we discuss the convergence to a stationary point of \revise{the \name\ algorithm} (Algorithm  \ref{alg:BCD_NMD}), which, to the best of our knowledge, has not been studied before \revise{in the context of RMD}. 
Indeed, our goal is to prove that \revise{\name} is convergent under the assumption that a limit point of the sequence generated by the scheme exists. We proceed by showing that the method
falls into the more general framework of exact BCD, where one of the subproblems is  strictly quasi-convex.  
Specifically, the general BCD scheme is studied by Grippo \revise{and Sciandrone}~\cite{grippo2000convergence} when the objective function is continuously differentiable over the Cartesian product of $m$ closed, convex sets, one for each block of variables. We first observe that the \revise{3B-RMD} objective function is continuously differentiable and defined over three blocks of variables $Z$, $W$, and $H$. The latent variable $Z$ is the only constrained variable, and it is easy to show that the constraint set in (\ref{eq:conv_set_Z}) is convex. Therefore, \revise{3B-RMD} is a particular instance of the more general problem considered in~\cite{grippo2000convergence}. 

\begin{theorem}
Any limit point of the sequence generated by \revise{\name} 
(Algorithm~\ref{alg:BCD_NMD}) is a stationary point of problem (\ref{eq:lat_3B_matr_com}).
\label{th:conv_bcd}
\end{theorem}
\begin{proof}
The global convergence of the BCD scheme for a general function with $m$ blocks of variables is analyzed under generalized convexity assumptions in~\cite{grippo2000convergence} (Proposition 5). The key hypotheses are: 1) the existence of a limit point of the sequence generated by the BCD algorithm, and 2) strict quasi-convexity of at least $m-2$ subproblems. Therefore, since \revise{3B-RMD} has three blocks of variables, we just need  one subproblem to be strictly quasi-convex. The first subproblem in the variable $Z$ is a convex, quadratic problem over a convex set, making it strictly quasi-convex, implying the convergence \revise{to a stationary point}. 
\end{proof}

The assumption of the existence of the limit point is an undesirable hypothesis. However, it is necessary because the level sets of the objective function are not bounded. Therefore, similarly to \revise{Latent-RMD}, \revise{3B-RMD} might have instances in which the minimum is not attained (see Theorem~\ref{th:equival_th}). This means that, in pathological cases, the sequence generated by \revise{\name} may diverge simply because the infimum of the function is not attained by any finite value. Thus, the hypothesis regarding the existence of a limit point of the \revise{\name} sequence cannot be weakened. One possible way to overcome this limitation is by adding a regularization term to the objective function~\eqref{eq:lat_3B_matr_com}, which makes the level sets compact, as in~\cite{wang2024momentum}. However, this is not the purpose of this work which focuses on the non-regularized formulation.  

\subsection{Extrapolated \name\ (\ename)}
\label{subsec:ebcd_nmd} 

\revise{We want} to accelerate the \revise{\name} algorithm while preserving convergence to a stationary point under mild assumptions. The \revise{\ename} method is an adaptation of the well-known LMaFit algorithm~\cite{wen2012solving} for matrix completion. 
The standard matrix completion formulation in~\cite{wen2012solving} is similar to the one we consider in (\ref{eq:lat_3B_matr_com}), although it does not include the inequality constraint $P_{\Omega^C}(Z)\le 0$, which models the crucial information that the missing entries are negative. Consequently, we derive a new optimization scheme, differing from the one in~\cite{wen2012solving}, as the update of $Z$ presents an additional nonlinear term; see equation (\ref{eq:eBCD-NMD-scheme}). 
Moreover, we relax the hypothesis in~\cite{wen2012solving} that requires $\{W^k\}$ and $\{H^k\}$ to remain full-rank for all iterations. 
We present two different versions of the \revise{\ename} algorithm that, starting from the same point, produce the same low-rank approximation $W^kH^k$, yet the factors $W^k$ and $H^k$ vary. 
In what follows, we omit the iteration index, $k$, to simplify the notation, as we focus on a single iteration of the \revise{\ename} algorithm. 
We assume we have a point $(Z,W,H)$ from which we perform one iteration. 

\subsubsection{First version of \ename} Let $\alpha\geq 1$, then the first version of the \revise{\ename} scheme uses the following update: given the iterate $(Z,W,H)$, it computes 
\begin{equation}
    \begin{aligned}
        & Z_{\alpha}=\alpha Z + (1-\alpha) W H, \\
        &\widehat{W}(\alpha)=Z_{\alpha} H^{\dagger}, \\
        & \widehat{H}(\alpha) 
        = (\widehat{W}(\alpha)^T)^{\dagger} Z_\alpha,  \\
                & \widehat{Z}(\alpha)=P_{\Omega}(X)+P_{\Omega^C}(\min(0,\widehat{W}(\alpha)\widehat{H}(\alpha))),
    \end{aligned}
    \label{eq:eBCD-NMD-scheme}  
\end{equation}
where $\left(\widehat{Z}(\alpha),\widehat{W}(\alpha),\widehat{H}(\alpha) \right)$ provides the next iterate. Similarly to \revise{\name}, the main computational cost for the scheme in (\ref{eq:eBCD-NMD-scheme}) is solving the two least squares problems for computing $\widehat{W}(\alpha)$ and $\widehat{H}(\alpha)$.
The optimization scheme in (\ref{eq:eBCD-NMD-scheme}) differs from the one in (\ref{eq:BCD_NMD}) because of the new variable $Z_{\alpha}$ which is used in place of the original latent variable $Z$ to update $W$ and $H$. Moreover, if we assume the matrix $H$ to be full-rank, we can derive an intuitive interpretation of the update of $Z_\alpha$ as an \textit{indirect extrapolation step}. Since $H$ is full-rank we have $H^\dagger=H^T(HH^T)^{-1}$ and thus
$$
    \widehat{W}(\alpha)=Z_{\alpha}H^T(HH^T)^{-1}=\alpha Z H^T(HH^T)^{-1}+(1-\alpha) WH H^T(HH^T)^{-1} = \alpha W_{\text{BCD}} + (1-\alpha) W,
    \label{eq:extr_BCD_seq}
$$
where $W_{\text{BCD}}$ is the update of the variable $W$ that we would have had if we performed one step of the \revise{\name} scheme (\ref{eq:BCD_NMD}). Moreover, if we define $\beta=\alpha-1$, then  
\begin{equation}
    \widehat{W}(\alpha)= \alpha W_{\text{BCD}} + (1-\alpha) W= (1+\beta) W_{\text{BCD}}-\beta W = W_{\text{BCD}}+\beta( W_{\text{BCD}}-W).
    \label{eq:extr_interp}
\end{equation}
\revise{Thus, line~\eqref{eq:extr_interp} represents} an extrapolation step for \revise{$W_{\text{BCD}}$}, provided that $\alpha \in (1,2)$, that means $\beta \in (0,1)$.
Unfortunately, such an intuitive interpretation does not apply to the update of the variable $H$. Note that the step in (\ref{eq:extr_interp}) is properly referred to as an extrapolation step if $\beta$ is in the interval (0,1) or equivalently, $\alpha$ is in the interval (1,2), but our scheme potentially allows for larger values that exceed the upper bound. 

\subsubsection{Improved version of \ename} The second version of the \revise{\ename} scheme does not require solving any least squares problem and the main computational cost is associated with a single QR factorization of a matrix in $\mathbb{R}^{m \times r}$. We specify that this scheme can be adapted to design alternative versions of \revise{\name} and \revise{e3B} as well. We state the following lemma from~\cite{wen2012solving} that \revise{is used in what follows.}
\begin{lemma}    \label{lem:equival_range_v2}
    Let $\widehat{W}(\alpha)$ be the update of $W$ for the scheme in (\ref{eq:eBCD-NMD-scheme}), then   
    $$
        \mathcal{R}(\widehat{W}(\alpha))=\mathcal{R}(Z_\alpha H^T), 
        \label{eq:range_eq}
    $$ 
    where $\mathcal{R}(A)$ denotes the 
range of matrix $A$. 
\end{lemma}
\begin{proof}
    Let $U D V^T=H$ be the economy form SVD of $H$, then 
        $\widehat{W}(\alpha)=Z_\alpha H^\dagger=Z_\alpha V D^\dagger U^T$.
    Moreover, since $Z_\alpha H^T=Z_\alpha V D U^T$, we get $ \mathcal{R}(\widehat{W}(\alpha))=\mathcal{R}(Z_\alpha H^T)$.
\end{proof}
Lemma~\ref{lem:equival_range_v2} shows that the range space of $\widehat{W}(\alpha)$ can be obtained from that of $Z_\alpha H^T$.
Let $Q(\alpha)$ be a matrix whose columns \revise{form} an \revise{orthonormal} basis for $\mathcal{R}(Z_\alpha H^T).$ Then, using Lemma~\ref{lem:equival_range_v2} and from the second and third lines in (\ref{eq:eBCD-NMD-scheme}), we get
\begin{equation}
    \widehat{W}(\alpha) \widehat{H}(\alpha)=\widehat{W}(\alpha) \widehat{W}(\alpha)^\dagger Z_\alpha = Q(\alpha) Q(\alpha)^T Z_\alpha.
    \label{eq:orth_relation}
\end{equation}
This shows that the low-rank approximation produced by the \revise{\ename} is nothing else than the orthogonal projection of $Z_\alpha$ into $\mathcal{R}(Z_\alpha H^T)$. 
Justified by the relation in (\ref{eq:orth_relation}), we now state the improved version of \revise{\ename}. Given an iterate $(Z,W,H)$, we compute
\begin{equation}
    \begin{aligned}
        & Z_{\alpha}=\alpha Z + (1-\alpha) W H, \\
        & W(\alpha)=Q(\alpha) \revise{\ \leftarrow} \text{ \revise{Orthonormal} basis of } \mathcal{R}(Z_\alpha H^T), \\
        &H(\alpha) =  W(\alpha)^T Z_\alpha= Q(\alpha)^T Z_\alpha,\\
        & Z(\alpha)=P_{\Omega}(X)+P_{\Omega^C}(\min(0,W(\alpha)H(\alpha))),
    \end{aligned}
    \label{eq:eBCD-NMD-scheme_v2}  
\end{equation}
to obtain the next iterate $\left(Z(\alpha), W(\alpha), H(\alpha)\right)$. Equation (\ref{eq:orth_relation}) guarantees that the low-rank approximations computed by the scheme in (\ref{eq:eBCD-NMD-scheme}) and in (\ref{eq:eBCD-NMD-scheme_v2}) are equal, that is $\widehat{W}(\alpha) \widehat{H}(\alpha)=W(\alpha)H(\alpha)$. In practice, $Q(\alpha)$ can be computed from the economy QR decomposition of $Z_\alpha H^T$. Note that if $\rank(Z_\alpha H^T)=r_1 < r$, the first $r_1$ columns of the $Q$ matrix from the economy QR decomposition of $Z_\alpha H^T$ are not, in general, an \revise{orthonormal} basis of the range space of the matrix. To ensure this property, a modified version of the QR decomposition using column pivoting must be used instead. Specifically, there exists $Q(\alpha) \in \mathbb{R}^{m \times r_1}$, containing an \revise{orthonormal} basis of $Z_\alpha H^T$, such that $Z_\alpha H^T \Pi=Q(\alpha) R(\alpha)$, where $\Pi$ is a permutation matrix and $R(\alpha) \in \mathbb{R}^{r_1 \times r}$~\cite{golub2013matrix}. Therefore, we potentially allow for a drop in the rank of the approximation whenever $Z_\alpha H^T$ is rank-deficient. However, we have never encountered this case  in practical situations because  $Z_\alpha H^T$ is likely to have rank at least $r$ while the factors $Q(\alpha)$ and $H(\alpha)$ have rank exactly $r$. 

 We proceed by defining the residual of our algorithm, denoted as $S(\alpha)$, 
\begin{equation}
        S(\alpha)= Z(\alpha)-W(\alpha)H(\alpha)=P_{\Omega}(X-W(\alpha)H(\alpha))-P_{\Omega^C}(\max(0,W(\alpha)H(\alpha))).
    \label{eq:residual}
\end{equation}
We highlight an interesting relation between the variable $Z_\alpha$ and the residual. Let $S$ denote the residual matrix of the iterate $(Z,W,H)$, defined as in (\ref{eq:residual}) with $Z$, $W$, and $H$ instead of $Z(\alpha)$, $W(\alpha)$ and  $H(\alpha)$ respectively, then
\begin{equation}
    \begin{aligned}
         Z_{\alpha}&=\alpha Z + (1-\alpha)WH= \alpha (P_{\Omega}(X)+P_{\Omega^C}(\min(0,WH))-WH)+WH\\
       &=\alpha(P_{\Omega}(X-WH)-P_{\Omega^C}(\max(0,WH)))+WH\\
       &=WH+\alpha S.
    \end{aligned}
    \label{eq:res_corr_Z}
\end{equation}
Equation (\ref{eq:res_corr_Z}) shows that $Z_\alpha$ is indeed the low-rank approximation $WH$ at step $k$, to which we add a multiple of the residual $S$. This suggests a second interpretation of the update of $Z_\alpha$ as a residual corrected step of the low-rank approximation $WH$. 
\paragraph{Restarting scheme of the extrapolation parameter} We need to determine a rule to select the extrapolation parameter $\alpha$ at each iteration. Extrapolating a variable means adding information from the past iterate; thus, the extrapolation parameter $\alpha$ affects \revise{the relevance of} the past information in the current update. Of course, not all the values of $\alpha$ are acceptable to progress in minimizing the residual, and a good choice of the parameter is crucial. The updating scheme of the extrapolation parameter $\alpha$, is inspired by the one introduced in~\cite{wen2012solving}. The main idea is to choose the parameter $\alpha$ such that the residual of the function is sufficiently decreased between two consecutive iterations.  

 At each step, we start from a candidate value for $\alpha$, and evaluate the ratio $\delta(\alpha)$ between the norms of the residual $S(\alpha)$ at iteration $k+1$ and $S$ at iteration $k$, defined as 
\begin{equation}
    \delta(\alpha)=\frac{\lVert S(\alpha)\rVert_F}{\lVert S \rVert_F}. 
    \label{eq:res_ratio}
\end{equation}
If $\delta(\alpha)<1$, we declare our update successful, and the new iterate is set to $(W^{k+1},H^{k+1},Z^{k+1})=(W(\alpha),H(\alpha),Z(\alpha))$. Furthermore, let $\Bar{\delta}$ be a fixed threshold, if $ \delta(\alpha)>\Bar{\delta} \in (0,1)$, meaning we do not consider the decrease of the residual fast enough, we increase $\alpha$ by setting $\alpha \leftarrow \min(\alpha+\mu, \Bar{\alpha})$, where $\mu=\max(\mu,0.25(\alpha-1))$. If $\alpha$ reaches the upper bound $\Bar{\alpha}$, we restart the parameter by setting $\alpha=1$. This correction does not appear in~\cite{wen2012solving}, but we added it to guarantee the convergence of our algorithm. On the contrary, if  $\delta(\alpha) \geq 1$, the step is declared unsuccessful, we do not accept the update, and we compute a standard \revise{\name} update instead by setting $\alpha= 1$. Consequently, either we find an $\alpha \in (1,\Bar{\alpha})$ that allows a decrease of the residual, or we recover one iteration of the \revise{\name}, which is guaranteed to decrease the residual. Therefore, the new iterate of the algorithm produces a lower value in the residual with respect to the previous iteration. 

The overall \revise{\ename} scheme is described in Algorithm \ref{alg:eBCD_NMD}, where we consider the updating scheme in (\ref{eq:eBCD-NMD-scheme_v2}). We choose the improved version of \revise{\ename} because (a) it exhibits faster performance in practice; (b) computing one QR factorization is, in general, more stable than solving two consecutive least squares problems; (c) one of the factors in the decomposition is orthogonal and this helps demonstrate crucial properties in the convergence analysis, such as boundness of the sequence generated by the \revise{\ename} algorithm. Let us also mention that,  in~\cite{wen2012solving}, convergence is proved for a scheme similar to the one in (\ref{eq:eBCD-NMD-scheme}), instead of the scheme in (\ref{eq:eBCD-NMD-scheme_v2}) as done in this paper. 

\begin{algorithm}[h]
 \caption{eBCD for 3B-RMD}
\begin{algorithmic}[1]
        \REQUIRE $X$, $Z^0$,$W^0$, $H^0$, $\Bar{\alpha}$  (default $= 4$), 
        $\mu$ (default $= 0.3$),
        $\bar \delta$  (default $= 0.8 $), maxit
        \STATE Set $\Omega=\{(i,j) \ | \ X_{ij}>0 \}$ and $\alpha_0=1$. 
        \FOR{$k=0,\dots,$ maxit } 
            \STATE $Z_{\alpha_k}=\alpha_k Z^k + (1-\alpha_k)W^kH^k$
            \STATE Compute $Q(\alpha_k)$ \revise{orthonormal} basis of $\mathcal{R}(Z_{\alpha_k} (H^k)^T)$
            \STATE $W(\alpha_k)=Q(\alpha_k)$
            \STATE $H(\alpha_k)=Q(\alpha_k)^T Z_{\alpha_k}$
            \STATE $Z(\alpha_k)=P_{\Omega}(X)+P_{\Omega^C}(\min(0,W(\alpha_k)H(\alpha_k))) $
            \vspace{0.1cm}
            \IF{$\delta(\alpha_k) \geq 1$}
            \STATE Set $\alpha_k$ to 1

            \STATE $(W^{k+1},H^{k+1},Z^{k+1})=(W^{k},H^{k},Z^{k})$ 
            \ELSE
                \STATE $(W^{k+1},H^{k+1},Z^{k+1})=(W(\alpha_k),H(\alpha_k),Z(\alpha_k))$ 
                \IF{$\delta(\alpha_k) \geq \Bar{\delta}$}
                    \STATE $\mu=\max(\mu,0.25(\alpha_k-1))$ and $\alpha_{k+1}=\min(\alpha_k+\mu, \Bar{\alpha})$
                        \IF{$\alpha_{k+1} = \Bar{\alpha}$}
                            \STATE $ \alpha_{k+1} = 1,$
                        \ENDIF
                    \ELSE 
                    \STATE $\alpha_{k+1}=\alpha_k.$
                \ENDIF
        \ENDIF
        \ENDFOR
\end{algorithmic}
\label{alg:eBCD_NMD}
 \end{algorithm}

\subsubsection{Convergence analysis} 
We present now the convergence analysis of Algorithm \ref{alg:eBCD_NMD}. The proofs of the lemmas can be found in the Appendix.
The main steps of the convergence analysis are summarized as follows. 
\begin{enumerate}
    \item[Step 1] We first prove that there exists a range of values larger than 1 for $\alpha$ such that the scheme (\ref{eq:eBCD-NMD-scheme_v2}) decreases the residual, that is, 
        $\lVert S \rVert_F^2-\lVert S(\alpha) \rVert_F^2 >0$. 
    Specifically, we split this residual reduction into the contribution given by the update of each variable separately: Lemma~\ref{lem:sig_WH} characterizes the residual reduction after the update of $W$ and $H$, showing that it is nonnegative. Lemma~\ref{lem:sigm_pc_z} and Lemma~\ref{lem:sig_Z} describe the residual reduction after updating the entries of $Z$. 
    
    \item[Step 2] We provide the KKT conditions for problem 
    \eqref{eq:lat_3B_matr_com}; see \eqref{eq:KKT_NMD}. 
    
    \item[Step 3] We use the residual reduction that we characterized in Step 1 to prove the subsequence convergence of the \revise{\ename} algorithm in Theorem \ref{th:conv_ebcd}, showing that a subsequence of the \revise{\ename} iterates satisfies the KKT conditions at the limit. This is done by analyzing all possible situations that the restarting scheme arises. 
\end{enumerate}

\paragraph{Step 1} 

Let us denote $\sigma_{WH}(\alpha)$ the variation of the residual after the update of $W$ and $H$ and $\sigma_{Z}(\alpha)$ the variations after updating the entries of $Z$,  that is, 
$$
    \lVert S \rVert_F^2-\lVert S(\alpha) \rVert_F^2=\underbrace{\lVert S \rVert_F^2-\frac{1}{\alpha^2} \lVert W(\alpha)H(\alpha)-Z_{\alpha} \lVert_F^2}_{=:\sigma_{WH}(\alpha)} + \underbrace{\frac{1}{\alpha^2} \lVert W(\alpha)H(\alpha)-Z_{\alpha} \lVert_F^2 -\rVert S(\alpha) \rVert_F^2}_{=: \sigma_{Z}(\alpha)}
    \label{eq:res_decr}
$$

Let $U$ be an \revise{orthonormal} basis of $\mathcal{R}(H^T)$. If $H$ is full rank, we can define the orthogonal projections onto $\mathcal{R}(W(\alpha))$ and $\mathcal{R}(H^T)$ as:
\begin{equation}
    E(\alpha) = W(\alpha)W(\alpha)^T=Q(\alpha)Q(\alpha)^T, \qquad 
    P = UU^T = H^T(HH^T)^{-1}H.  
    \label{eq:orth_proj}
\end{equation} 
The following three lemmas characterize $\sigma_{WH}(\alpha)$ and $\sigma_{Z}(\alpha)$ (the proofs can be found in the appendix). 
\begin{lemma} \label{lem:sig_WH} 
Given $(Z,W,H)$ where $H$ is full rank, let  $(W(\alpha),H(\alpha))$ be computed by the scheme (\ref{eq:eBCD-NMD-scheme_v2}), 
then for any $\alpha \geq 1$, 
    \begin{equation}
        \sigma_{WH}(\alpha)=\frac{1}{\alpha^2} \lVert W(\alpha) H(\alpha) - WH \rVert_F^2=\lVert SP \rVert_F^2+\lVert E(\alpha) S(I-P) \rVert_F^2 \geq 0,
        \label{eq:sigma_12}
    \end{equation}
    where $E(\alpha)$ and  $P$ are defined  in (\ref{eq:orth_proj}). 
\end{lemma}
Therefore, after the first two steps, the residual reduction is strictly positive unless $W(\alpha)H(\alpha)=WH$. 
We look at the residual reduction after updating the entries of $Z$ in $\Omega^C$ (Lemma~\ref{lem:sigm_pc_z}) and in $\Omega$ (Lemma~\ref{lem:sig_Z}). In the following lemmas, we assume that there exists an interval $[1,\widehat{\alpha}]$, such that $Z_\alpha H^T$ and $Z H^T$ have the same rank. This hypothesis is needed to ensure the continuity of the residual reduction $\sigma_Z(\alpha)$ in a range of $\alpha>1$; more details are given in the Appendix.
\begin{lemma}
     Let $(W(\alpha),H(\alpha))$ be generated by  the scheme in (\ref{eq:eBCD-NMD-scheme_v2}). Furthermore, assume there exists $\widehat{\alpha}>1$ such that $\rank(Z_{\alpha}H^T)=\rank(ZH^T)$ for  $\alpha \in [1,\widehat{\alpha}]$, then there exists $\alpha_*>1$ such that for every $\alpha \in (1,\min(\alpha_*,\widehat{\alpha})]$, 
     \begin{equation}
         \frac{1}{\alpha^2} \lVert P_{\Omega^C}(W(\alpha)H(\alpha)-Z_{\alpha}) \rVert_F^2-\lVert P_{\Omega^C}(S(\alpha)) \rVert_F^2 \geq 0.
         \label{eq:sigma_z_pc}
     \end{equation}   
     Moreover, if $P_{\Omega^C}(\min(0,W(1)H(1)))\neq P_{\Omega^C}((\min(0,WH))$, a strict inequality holds.
     \label{lem:sigm_pc_z}
\end{lemma}

A similar result is proved in \cite{wen2012solving} for  matrix completion. In their less restrictive setting, the residual reduction after updating the entries of $Z$ in $\Omega^C$ is independent of $\alpha$ and nonnegative. However, due to the additional constraint $P_{\Omega^C}(Z) \leq 0$ in the \revise{3B-RMD} formulation, the residual reduction in (\ref{eq:sigma_z_pc}) of the \revise{\ename} algorithm is only locally nonnegative in a range where $\alpha$ is larger than one.  

Analogously to Lemma~\ref{lem:sigm_pc_z}, we now state a similar result that characterizes the residual reduction when the entries of $Z$ in $\Omega$ are updated.
\begin{lemma}     \label{lem:sig_Z} 
  Let $(W(\alpha),H(\alpha))$ be generated by the scheme in (\ref{eq:eBCD-NMD-scheme_v2}). Furthermore, assume there exists $\widehat{\alpha}>1$ such that rank($Z_{\alpha}H^T$)=rank($ZH^T$), for  $\alpha \in [1,\widehat{\alpha}]$. Then 
   \begin{equation} \label{eq:sigma_z} 
       \lim_{\alpha \to 1^+}\frac{1}{\alpha^2}\lVert P_{\Omega}(W(\alpha)H(\alpha)-Z_{\alpha}) \rVert_F^2-\lVert P_{\Omega}(S(\alpha)) \rVert_F^2 = 0. 
   \end{equation}  
\end{lemma} 
The following corollary gathers the findings from Lemma~\ref{lem:sig_WH}, Lemma~\ref{lem:sigm_pc_z}, and Lemma~\ref{lem:sig_Z}.
\begin{corollary}  \label{cor:res_red}  
    Let $(W(\alpha),H(\alpha),Z(\alpha))$ be generated by  the scheme in (\ref{eq:eBCD-NMD-scheme_v2}). Furthermore, assume there exists $\widehat{\alpha}>1$ such that rank($Z_{\alpha}H^T$)=rank($ZH^T$) for  $\alpha \in [1,\widehat{\alpha}]$.
     If $P_{\Omega^C}(\min(0,WH)) \neq P_{\Omega^C}(\min(0,W(1) H(1)))$, then there exists $\Tilde{\alpha} >1$ such that 
    \begin{equation}           \label{eq:res_diff}  
        \lVert S \rVert_F^2-\lVert S(\alpha) \rVert_F^2=\sigma_{WH}(\alpha)+\sigma_{Z}(\alpha)>0 \quad \forall \alpha \in [1,\min(\widehat{\alpha},\Tilde{\alpha})].
    \end{equation}
\end{corollary}
Corollary~\ref{cor:res_red} guarantees that there exist some values of the extrapolation parameter $\alpha$ strictly larger than 1 resulting in a positive residual reduction 
for \revise{\ename}~(\ref{eq:eBCD-NMD-scheme_v2}). 

\paragraph{Step 2} 

Before stating the actual convergence theorem, we need to introduce the optimality conditions or KKT conditions for problem (\ref{eq:lat_3B_matr_com}). Let us define the Lagrangian function as
\begin{equation}
    \mathcal{L}(W,H,Z,\Lambda,\Sigma)=\frac{1}{2} \lVert  Z-WH \rVert_F^2+\langle \Lambda, P_{\Omega}(Z-X) \rangle +\langle \Sigma, P_{\Omega^C}(Z) \rangle,
    \label{eq:lagrang_NMD}
\end{equation}
where the dual variables are such that $\Lambda=P_{\Omega}(\Lambda)$ and $\Sigma=P_{\Omega^C}(\Sigma)$:
\begin{equation}
    \begin{aligned}
         \text{Stationarity condition: } &\nabla_W \mathcal{L}(Z,W,H) = -(Z-WH)H^T = 0,\\
       &\nabla_H \mathcal{L}(Z,W,H) =-W^T(Z-WH) = 0,\\
        &\nabla_Z \mathcal{L}(Z,W,H,\Lambda,\Sigma) =Z-WH+\Lambda+\Sigma=0, \\  
        \text{Primal feasibility: } & P_{\Omega}(Z-X)=0, P_{\Omega^C}(Z) \leq 0, \\      
        \text{Dual feasibility: }  &\Sigma \geq 0,\\
        \text{Complementary slackness: }&  \langle \Sigma, P_{\Omega^C}(Z) \rangle=0, 
    \end{aligned}
    \label{eq:KKT_NMD}
\end{equation}
where the multiplier matrices $\Lambda$ and $\Omega$ measure the residual in $\Omega$ and $\Omega^C$ respectively, that is,  $\Lambda=P_{\Omega}(WH-Z)$ and $\Sigma=P_{\Omega^C}(\max(0,WH))$. Our goal is to prove that, under some mild assumptions, a subsequence of the sequence generated by \revise{\ename} satisfies the KKT condition at the limit. Specifically, one can verify that the \revise{\ename} sequence satisfies the complementary slackness condition and the stationarity condition for $Z$ at any point. 
Therefore, we just need to show that the optimality conditions for $W$ and $H$ are satisfied by a subsequence of the \revise{\ename} scheme at the limit, that is, 
\begin{equation}
    \lim_{k \to \infty} \nabla_W \mathcal{L}(Z^k,W^k,H^k)=0,  \qquad \lim_{k \to \infty} \nabla_H \mathcal{L}(Z^k,W^k,H^k)=0.
    \label{eq:opt_cond_kkt}
\end{equation}

\paragraph{Step 3} 
We are now ready to prove the convergence result for Algorithm  \ref{alg:eBCD_NMD}.

\begin{theorem}
\label{th:conv_ebcd}
    Let $(W^{k+1},H^{k+1},Z^{k+1})$ be generated by \revise{eBCD} (Algorithm~\ref{alg:eBCD_NMD}), and assume that $\{P_{\Omega^C}(W^k H^k)\}$ is bounded for all $k$.  
    Then there exists a bounded subsequence of \\  $(W^{k+1},H^{k+1},Z^{k+1})$ whose limit point satisfies the KKT conditions (\ref{eq:KKT_NMD}). 
\end{theorem} 
\begin{proof} We first discuss the boundness of the two sequences  $\{W^{k+1}\}$ and $\{H^{k+1}\}$ produced by Algorithm~\ref{alg:eBCD_NMD}. Te boundness of $\{P_{\Omega^C}(W^k H^k)\}$ implies that $\{Z^k\}$ and $\{W^k H^k \}$ are bounded. This also means that $\{Z_{\alpha_k}\}$ is a bounded sequence by construction. Moreover, the sequence $\{W^{k+1}\}$ contains orthogonal matrices whose columns \revise{form an orthonormal} basis for $\mathcal{R}(Z_{\alpha_k} (H^k)^T)$, thus it is bounded. Finally, $H^{k+1}=(W^{k+1})^T Z_{\alpha_k}$ and since $\{W^{k+1}\}$ and $\{Z_{\alpha_k}\}$ are bounded, then also $\{H^{k+1}\}$ is bounded.

We observe that at every iteration $H^{k+1}$ has always the same rank as $W^{k+1}$. This means that $H^{k+1}$ is full rank for every $k$ and Lemma~\ref{lem:sig_WH} can be applied. In fact, it follows from Lemma~\ref{lem:equival_range_v2} that $\mathcal{R}(Q^{k+1})=\mathcal{R}(Z_{\alpha_k}(H^k)^T) \subseteq \mathcal{R}(Z_{\alpha_k})$; that implies
$$\mathcal{R}(I)=\mathcal{R}((Q^{k+1})^TQ^{k+1}) \subseteq \mathcal{R}((Q^{k+1})^T Z_{\alpha_k})=\mathcal{R}(H^{k+1}),$$
where $I$ is the identity matrix of dimension equal to the number of columns of $Q^{k+1}$ or equivalently, to the rank of $Z_{\alpha_k}(H^k)^T$.

   Next, we define the set of indices where the update of $Z^k$ results in a positive residual reduction. Let $\mathcal{F}= \left\{ k \ : \ \sigma_{Z}(\alpha_k) \geq 0 \right\} $ and $\mathcal{F}^C$ its complement. 
   Therefore, denoting $E^k=E(\alpha_k)$,  since $H^{k+1}$ is full rank, we can use Lemma~\ref{lem:sig_WH}, and summing over all $k \in \mathcal{F}$, we get
   \begin{equation}
       \lVert S^0 \rVert_F^2 \geq \sum_{k \in \mathcal{F}} \sigma_{WH}(\alpha_k)= \sum_{k \in \mathcal{F}} \lVert S^k P^k \rVert_F^2+\lVert E^k S^k(I-P^k)\rVert_F^2.
       \label{eq:fond_eq_conv}
   \end{equation}
   We analyze separately two different scenarios. 

Case 1: assume $| \mathcal{F}^C | < \infty$. It follows from (\ref{eq:fond_eq_conv}) that 
   \begin{equation}
       \lim_{k \to \infty, k \in \mathcal{F}} \lVert S^k P^k \rVert_F^2=0, \quad  \lim_{k \to \infty, k \in \mathcal{F}} \lVert E^k S^k(I-P^k)\rVert_F^2=0,
   \label{eq:lim_to_zero_sq}
   \end{equation}
which implies
       $\lim_{k \to \infty, k \in \mathcal{F}} \lVert E^k S^k\rVert_F^2=0$. 
Recall that $P_{\Omega^C}(Z^k)=P_{\Omega^C}(\min(0,W^k H^k))$ and $P_{\Omega}(X)=P_{\Omega}(Z^k)$. 
Using (\ref{eq:orth_proj}), we have 
    $S^k P^k =(Z^k-W^k H^k)P^k=(Z^k-W^k H^k)U^k (U^k)^T$.
   Since $U^k$ is an orthonormal basis for $\mathcal{R}((H^k)^T)$ and $\{H^k\}$ bounded, 
   $$
       \lim_{k \to \infty, k \in \mathcal{F}}(Z^k-W^k H^k)(H^k)^T=-\nabla_W \mathcal{L}(Z^{k},W^{k},H^{k})=0.
       \label{eq:lim_to_zero}
   $$ 
Moreover,
\begin{equation}
    E^k S^k=E^k S^{k+1}+E^k(S^k-S^{k+1})=W^{k+1} (W^{k+1})^T(Z^{k+1}-W^{k+1}H^{k+1})+E^k(S^k-S^{k+1}).
    \label{eq:QS}
\end{equation}
   Furthermore
   \begin{equation}
   \begin{aligned}
       \lVert S^{k}-S^{k+1} \rVert_F^2=&\lVert P_{\Omega}(W^{k+1} H^{k+1}-W^k H^k) \rVert_F^2+\lVert P_{\Omega^C}(\max(0,W^{k+1} H^{k+1})-\max(0,W^k H^k)) \rVert_F^2\\
       \leq &\lVert W^{k+1} H^{k+1}-W^k H^k \rVert_F^2+\lVert \max(0,W^{k+1} H^{k+1})-\max(0,W^k H^k) \rVert_F^2\\
       \leq&2\lVert W^{k+1} H^{k+1}-W^k H^k \rVert_F^2.
         \end{aligned}
         \label{eq:S_Sold}
   \end{equation} 
   Using equation (\ref{eq:sigma_12}) and the fact that $\alpha_k \leq \Bar{\alpha}$ by construction, we have 
   \begin{equation}
       \lVert W^{k+1} H^{k+1} - W^kH^k \rVert_F^2\leq\Bar{\alpha}^2(\lVert S^kP^k \rVert_F^2+\lVert E^k S^k(I-P^k) \rVert_F^2).
       \label{eq:WH_WHold}
   \end{equation}
   Combining (\ref{eq:WH_WHold}) and (\ref{eq:lim_to_zero_sq}) we get 
       $\lim_{k \to \infty, k \in \mathcal{F}} \lVert W^{k+1} H^{k+1} - W^kH^k \rVert_F^2=0$,
thus, from (\ref{eq:S_Sold}), 
   \begin{equation}
      \lim_{k \to \infty, k \in \mathcal{F}}  \lVert S^{k}-S^{k+1} \rVert_F^2=0.
      \label{eq:si_si_old}
   \end{equation}
   Therefore, from (\ref{eq:si_si_old}) and (\ref{eq:QS}), along with the boundness of $\{W^k\}$, we conclude that
   $$
      \lim_{k \to \infty, k \in \mathcal{F}} (W^{k+1})^T(Z^{k+1}-W^{k+1}H^{k+1})=\lim_{k \to \infty, k \in \mathcal{F}} -\nabla_H \mathcal{L}(Z^{k+1},W^{k+1},H^{k+1})=0.
   $$

   Case 2: $| \mathcal{F}^C |=\infty$.  Let us analyze two subcases. 
   
 \noindent  Case 2a: $| \{k \in \mathcal{F}^C \ | \ \delta(\alpha_k)>\Bar{\delta} \} | < \infty$. 
   For $k$ sufficiently large, 
       $\lVert S^{k+1} \rVert_F^2 \leq \Bar{\delta}  \lVert S^{k} \rVert_F^2$.  
   Therefore, $\lim_{k \to \infty, k \in \mathcal{F}^C} \lVert S^k \rVert_F^2=0,$ which means that the sequence converges to a global minimizer of the problem.

   \noindent   Case 2b: $| \{k \in \mathcal{F}^C \ | \ \delta(\alpha_k)>\Bar{\delta} \} | = \infty$. 
     \revise{\ename} sets the parameter $\alpha_k$ to 1  for an infinite number of iterations. Denote $\mathcal{J}_1$ the set of indices in which $\alpha_k=1$, a similar relation to (\ref{eq:fond_eq_conv}) holds: 
   $$
       \lVert S^0 \rVert_F^2 \geq \sum_{k \in \mathcal{J}_1} \sigma_{WH}(\alpha_k)= \sum_{k \in \mathcal{J}_1} \lVert S^k P^k \rVert_F^2+\lVert E^k S^k(I-P^k)\rVert_F^2.
       \label{eq:fond_eq_conv2}
   $$
   Therefore, one can replicate the proof for Case 1, and the subsequence whose indices are in $\mathcal{J}_1$ satisfies the optimality conditions at the limit.
\end{proof}

We emphasize that, similarly to \revise{\name}, we need a strong assumption, namely the boundness of $\{ P_{\Omega^C}(W^k H^k) \}$ which guarantees that the \revise{\ename} sequence remains bounded. This hypothesis serves the same purpose as the existence of the limit point in the \revise{\name} algorithm, since it allows us to exclude pathological cases where the infimum of the problem is not attained.

\section{Numerical experiments} 
\label{sec:num_res}

In this section, we present numerical experiments on synthetic and real-world data sets, for different applications of \revise{RMD}: matrix completion with ReLU sampling, the recovery of Euclidean distance matrices, low-dimensional embedding on text data, and the compression sparse dictionaries. 
Our main purpose is to show that \revise{\ename} behaves favorably compared to the state of the art, namely the \revise{Naive approach}~\cite{saul2022nonlinear}, \revise{EM}~\cite{saul2022nonlinear}, \revise{A-Naive}~\cite{seraghiti2023accelerated}, and \revise{e3B}~\cite{seraghiti2023accelerated}, while having convergence guarantees.  
In all experiments, we use random initialization: entries of $W$ and $H$ are sampled from a Gaussian distribution, \texttt{W = randn(m,r)} and \texttt{H = randn(r,n)} in MATLAB. 
We then scale them so that the initial point matches the magnitude of the original matrix, that is, we multiply $W$ by $ \sqrt{\lVert X \rVert_F} / \lVert W \rVert_F$ and $H$ by $ \sqrt{\lVert X \rVert_F}/ \lVert H \rVert_F$. 
If not otherwise specified, we stop each algorithm when we reach a small relative residual, namely $\Gamma_k=\lVert Z^k -W^kH^k\rVert_F/\lVert X\rVert_F \leq 10^{-9}$, or a fixed time limit is reached. 
The time limit will be selected depending on the size of the data. 
For the algorithms solving \revise{Latent-RMD}, the low-rank product $W^kH^k$ in $\Gamma_k$ is substituted with the matrix $\Theta^k$. The first stopping criteria can be achieved only if the target matrix admits an (almost) exact \revise{RMD}. The algorithms are implemented in MATLAB R2021b on a 64-bit Samsung/Galaxy with 11th Gen Intel(R) Core(TM) i5-1135G7 @ 2.40GHz and 8 GB of RAM, under Windows 11 version 23H2. The code is available online from 
\url{https://github.com/giovanniseraghiti/ReLU-NMD} .

The parameters of \revise{\ename} are the ones specified \revise{as default} in Algorithm~\ref{alg:eBCD_NMD}. For the parameters of the other methods, we use the setting from~\cite{seraghiti2023accelerated}. 

\subsection{Matrix completion with ReLU sampling}

 In this subsection, we consider the same experiment used in~\cite{liu2024symmetric} for symmetric matrix completion with ReLU sampling, adapted to the non-symmetric case. Specifically, we randomly generate the entries $W$ and $H$ from a Gaussian distribution, as for our initializations, and we construct $\Theta^t = W H$, and set $X=\max(0,\Theta^t)$ in the noiseless case and $\max(0,\Theta^t+N)$ in the noisy case, where $N=\sigma*\Tilde{N}*\lVert \Theta^t \rVert_F /\lVert \Tilde{N}\rVert_F $, $\Tilde{N} =$ \texttt{randn(m,n)}, and $\sigma \in (0,1)$. 
 In our experiment, we set $m=n=1000$, $r=20$  and $\sigma=10^{-2}$. 
 Note that since $W$, $H$ and $N$ are Gaussian, $X$ is approximately $50 \%$ sparse. 
 We present the loglog plot of the average residual $\Gamma_k$ against the average CPU time over 20 different randomly generated matrices $X$ in Figure \ref{Fig:box_plot_mc}. 
 The algorithms are stopped when the objective functions in (\ref{eq:lat_nmd}) and (\ref{eq:3B_nmd}) achieve a relative residual $\Gamma_k$ below $10^{-9}$ in the noiseless case,  or below $\sigma=10^{-2}$ in the noisy case. We also report the average CPU time \revise{and the standard deviation} to reach these values in Table~\ref{tab:ReLU_samp_cpu}. 
\begin{figure}[htb]
\begin{minipage}[b]{.45\linewidth}
  \centering
  \centerline{\includegraphics[width=\linewidth]{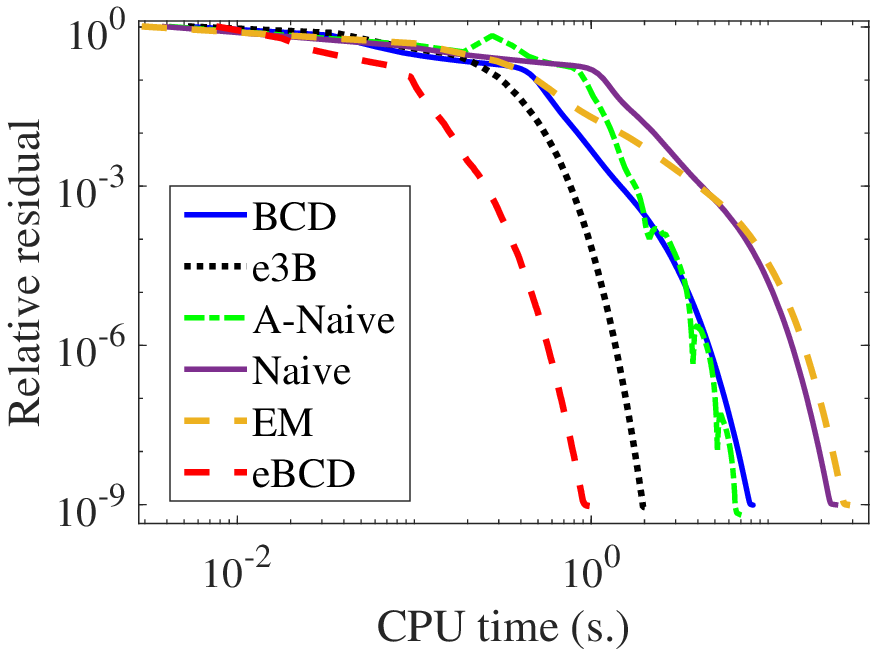}} 
\end{minipage}
\hfill
\begin{minipage}[b]{0.45\linewidth}
  \centering
  \centerline{\includegraphics[width=0.98\linewidth]{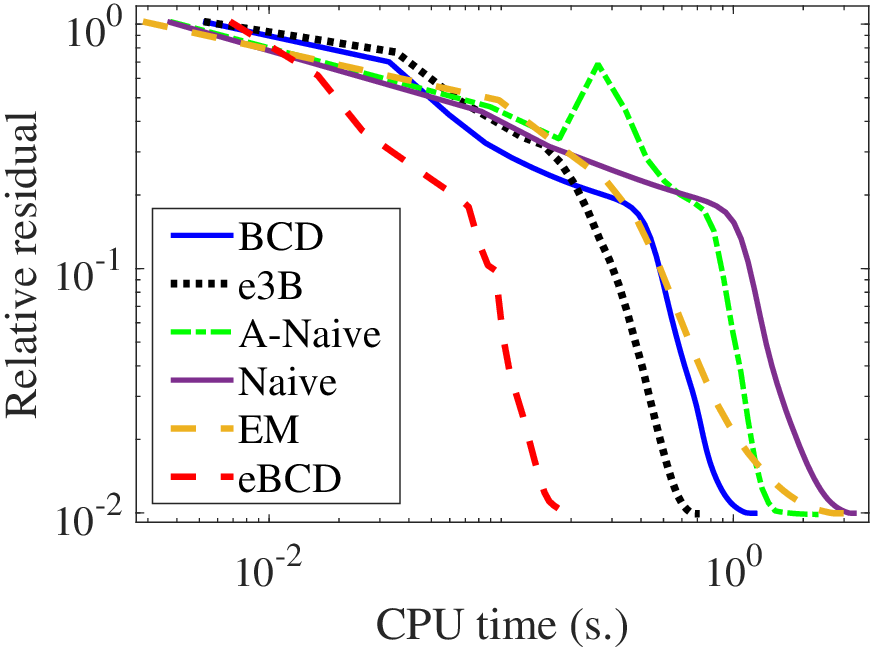}}
\end{minipage}
\caption{Matrix completion with ReLU sampling: evolution of the  average relative residual $\Gamma_k$ w.r.t.\ CPU time for the noiseless case (left image) and the noisy case (right image).} 
\label{Fig:box_plot_mc}
\end{figure}

\begin{table}[h!]
\centering
\begin{tabular}{ c | c | c | c | c | c | c | c } 
  & &\hspace{-0.25cm}\name\hspace{-0.4cm}&\hspace{-0.25cm}\ename\hspace{-0.4cm}&\hspace{-0.25cm}e3B\hspace{-0.4cm}&\hspace{-0.25cm}Naive\hspace{-0.4cm}& \hspace{-0.25cm}A-Naive\hspace{-0.4cm}&\hspace{-0.25cm}EM\hspace{-0.4cm}\\ \hline 
  \hspace{-0.55cm} No  \hspace{-0.3cm} \multirow{2}{*}{}& \hspace{-0.15cm} time \hspace{-0.15cm} & 7.92$\pm$0.16  & \textbf{0.90$\pm$0.04}     & 1.95$\pm$0.04  &   22.0$\pm$0.78 &   6.37$\pm$0.33 &   26.7$\pm$0.70 \\ 
             Noise    & \hspace{-0.15cm} iter \hspace{-0.15cm} & 304  & 121 & \textbf{65} & 308 & 84 & 296 \\ \hline   
   With  \multirow{2}{*}{}& \hspace{-0.15cm} time \hspace{-0.15cm} & 1.14$\pm$0.05& \textbf{0.19$\pm$0.01}   & 0.66$\pm$0.02 & 3.17$\pm$0.12 & 1.54$\pm$0.19 & 2.70$\pm$0.15\\ 
  Noise & \hspace{-0.15cm} iter \hspace{-0.15cm} & 36  & 22  & 20 & 41 & \textbf{19} & 28 \\ \hline   
\end{tabular}
\caption{Average CPU time in seconds, standard deviation, \revise{and average iteration number} to reach a tolerance on the relative residual $\Gamma_k$ of $10^{-9}$ in the noiseless case, and of $10^{-2}$ in the noisy case. Best values are highlighted in bold.}  
\label{tab:ReLU_samp_cpu}
\end{table}

Figure~\ref{Fig:box_plot_mc} shows that all the algorithms reach the tolerance on the relative residual. Moreover, \revise{\ename} outperforms all the other methods, being considerably faster both in the noiseless and noisy cases, as Table~\ref{tab:ReLU_samp_cpu} confirms. In fact, compared to the second best, \revise{\ename} is on average more than twice faster in the noiseless case, and more than three time faster in the noisy case (each time, the second best is \revise{e3B}). \revise{Table~\ref{tab:ReLU_samp_cpu} shows also that A-Naive and e3B need less iterations to converge than \ename. However, their computational cost per iteration is significantly larger.}

\subsection{EDMC}

We generate 200 points in two different ways in a three-dimensional space: (1)~uniformly at random in $[0,10]^3$, and (2)~randomly clustered around six clusters (4 with 30 points, and two with 40)  where the centroid $(\Bar{x},\Bar{y},\Bar{z})$ of the $i$th cluster in sampled uniformly at random in $[-10,10]^3$, and we  generate the points in MATLAB using $(\Bar{x},\Bar{y},\Bar{z})+\text{std}*\texttt{randn(1,3)}$, where std denotes the standard deviation that we fixed to std=3. 
We compute the Euclidean distance matrix $\Theta^t$ containing the squared distance between all couples of points and we construct our target matrix $X=\max(0,dee^T-\Theta^t)$, where $d$ represents the largest value in the matrix that we want to observe. The upper bound on the rank is fixed to $r=5$ which is equal to the rank of $\Theta^t$. In this experiment, we assume the threshold $d$ to be known, but one might be interested in considering $d$ as an unknown of the model. We aim to recover all the entries of $\Theta^t$ larger than $d$. In practice, the threshold $d$ is adapted to achieve the percentage of observed entries in the $x$-axis of the pictures on the right of Figure~\ref{Fig:EDMC_mc}. Specifically, we display the average relative error of the low-rank approximation with the Euclidean distance matrix $\Theta^t$ for different percentages of observed entries, over 10 random experiments. Moreover, all algorithms have been adapted to include the rank-1 modification in (\ref{eq:lat_3B_matr_com_rank1}). We excluded the \revise{EM} algorithm from the analysis because adapting it to the rank-1 modified model appears less straightforward than for the other algorithms. We also consider as a baseline the rank-$(r+1)$ reconstruction given by the original \revise{\name} method solving the \revise{3B-RMD} of rank $r+1$, without considering any rank-1 modification (($r+1$)-BCD). In this case, we compute a rank-($r+1$) decomposition of $X$ and evaluate the relative error with $dee^T-\Theta^t$. We stop the algorithms when a tolerance of $10^{-9}$ on the relative residual $\Gamma_k$ is reached or after 60 seconds. 

\begin{figure}[htb]
\begin{minipage}[b]{.45\linewidth}
  \centering
  \centerline{\includegraphics[width=\linewidth]{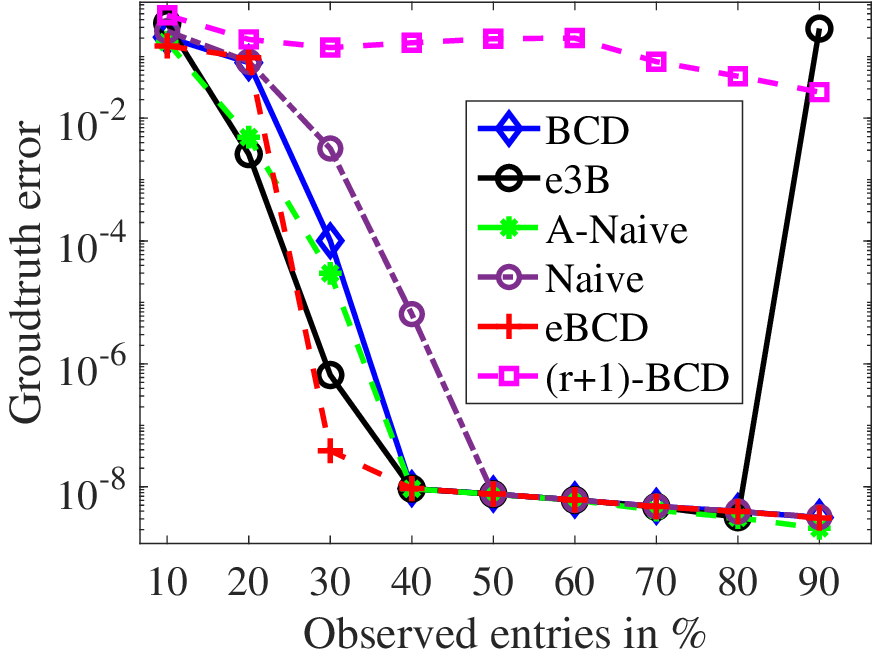}}
\end{minipage}
\hfill
\begin{minipage}[b]{0.45\linewidth}
  \centering
  \centerline{\includegraphics[width=\linewidth]{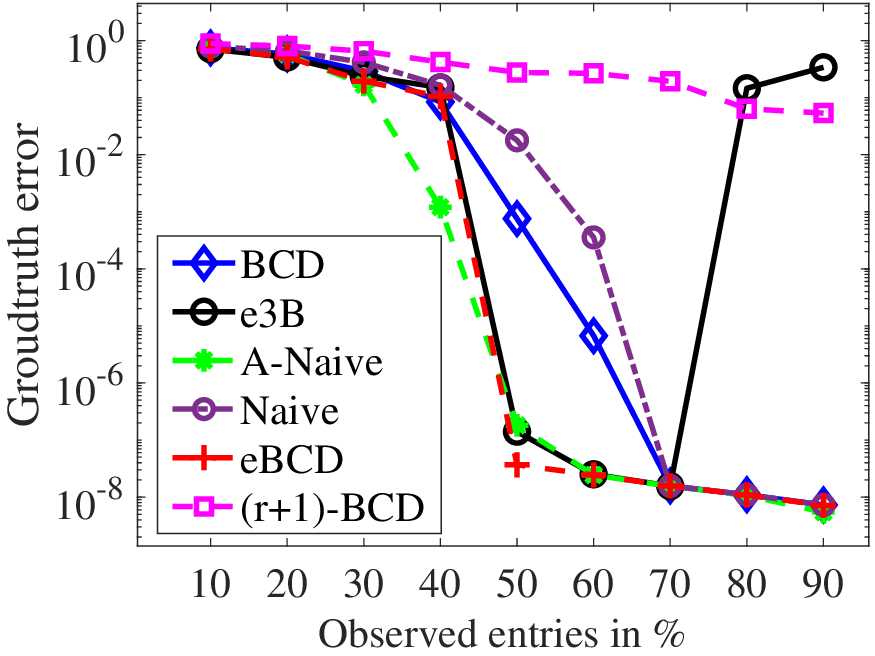}}
\end{minipage} 
\caption{Euclidean distance matrix completion: relative errors where the $x$-axis represents the percentage of known entries. The left figure is for the uniform distribution of the data points, the right figure for the clustered distribution.} 
\label{Fig:EDMC_mc}
\end{figure} 


The distribution of the points affects the recovery performance of the algorithms, see Figure~\ref{Fig:EDMC_mc}: 
    When the points are randomly distributed, $30\%$ of the entries need to be observed for the \revise{\ename} to achieve an error below $10^{-7}$. 
    When the points are clustered, only $50\%$ of the entries need to be observed for the \revise{\ename} to achieve an error below $10^{-7}$.  
Including the rank-1 modification in the model allows for a significantly more accurate reconstruction than solving a rank-($r+1$) \revise{3B-RMD}. In general, \revise{A-Naive}, \revise{\ename}, and \revise{e3B} provide the best reconstructions. However, \revise{e3B} fails at least once in both examples when the percentage of observed entries is above $80\%$. This behavior might be due to the aggressive extrapolation procedure that \revise{e3B} uses. We emphasize that neither the \revise{3B-RMD} nor the \revise{Latent-RMD} formulations impose any constraints to ensure that the solution is symmetric. However, if the number of observed entries is sufficiently large, the algorithms still recover the original symmetric solution.

\subsection{Low-dimensional embedding}

\revise{RMD} can be employed to find a lower-dimensional representation of the similarity matrix of the input points; see Section~\ref{sec:maniflearn}. 
Our focus is on text data, where each row of the data matrix corresponds to a document, each column to a word, and the entries $(i,j)$ of the matrix is the number of occurrences of the $j$th word in the $i$th document. We consider two well-known and already pre-pocessed text-word data sets from~\cite{zhong2005generative}:  \textit{\textit{\textit{k1b}}} with 2340 documents and 21839 words, 
and \textit{\textit{hitech}} with 2301 documents and 10080 words. 
The threshold parameter $\tau$ in the TSM framework is selected according to the statistical analysis proposed in~\cite{saul2022geometrical}, that is, $\tau=0.17$ for \textit{k1b} and $\tau=0.08$ for \textit{hitech}. 
The quality of the embedding is evaluated using the mean angular deviation (MAD) 
metric as suggested in~\cite{saul2022geometrical}:  
$$ \label{eq:mean_ang} 
    \Delta=\frac{1}{|\mathcal{P}(\tau)|} \sum_{(i,j) \in \mathcal{P}(\tau)} \left| \cos^{-1} \left( \frac{\langle z_i, z_j\rangle}{\lVert z_i \rVert \lVert z_j \rVert} \right) - \cos^{-1} \left( \frac{\langle y_i, y_j\rangle}{\lVert y_i \rVert \lVert y_j \rVert} \right)\right|,  
    $$
   where the $z_i$'s are the higher-dimensional points, and the $y_i$'s the lower-dimensional embeddings, and $\mathcal{P}(\tau)=\{(i,j)\ | \ \langle z_i,z_j \rangle -\tau \lVert z_j\rVert  \lVert z_i\rVert > 0\}$, 
Figure~\ref{Fig:rank_an_text} displays the MAD and average run times per iteration for different ranks. 
We run all the algorithms for 60 seconds. 
\begin{figure}[ht!]
\begin{center} 
\begin{tabular}{ccc}
  \includegraphics[width=0.4\linewidth]{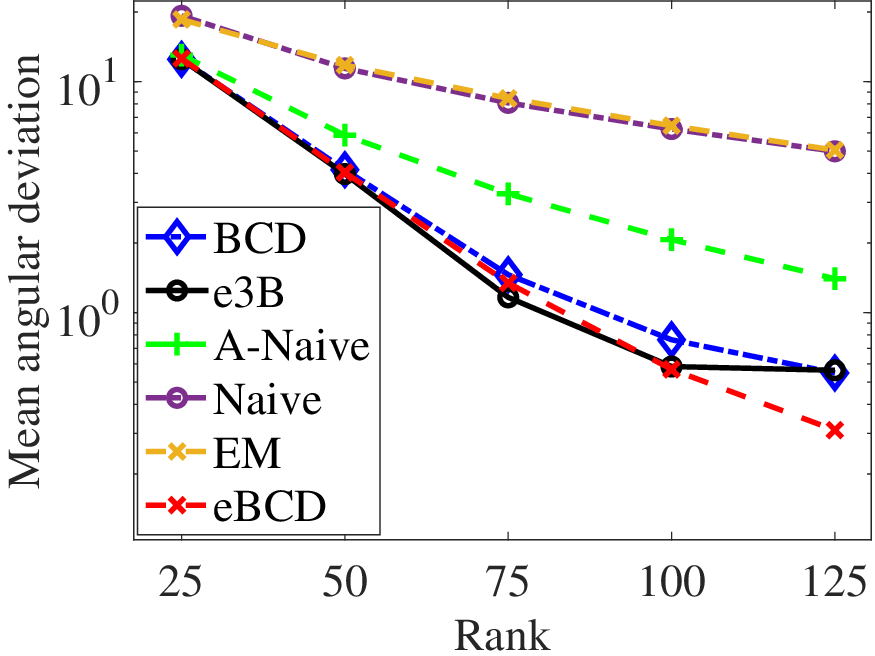}   & \quad & \includegraphics[width=0.4\linewidth]{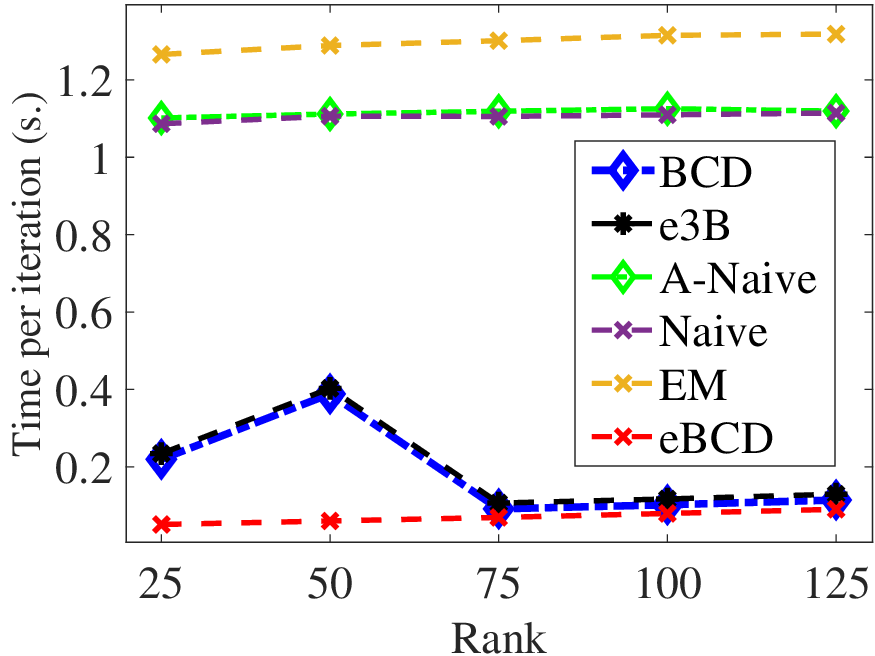} \\
   \includegraphics[width=0.4\linewidth]{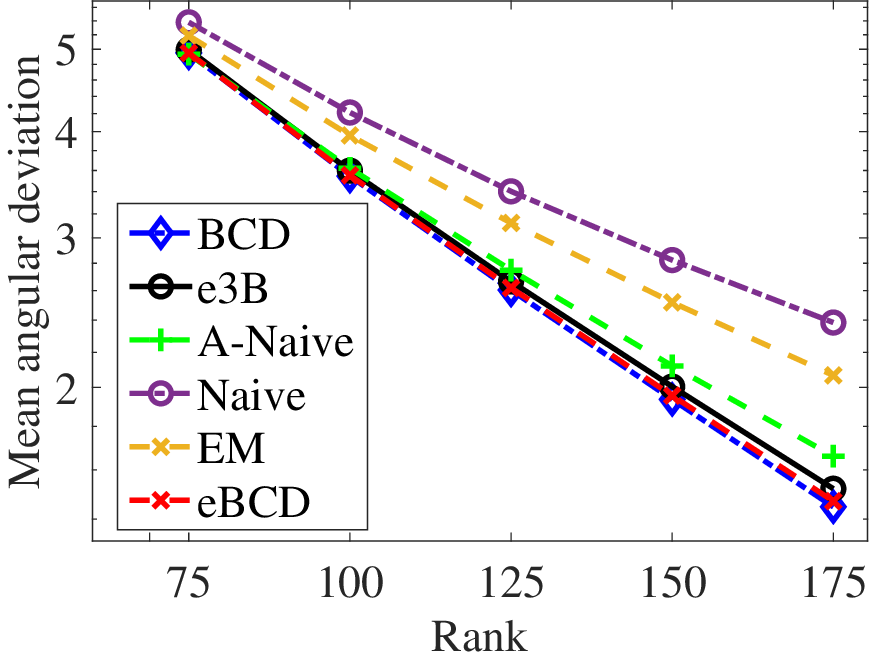}  & \quad & 
   \includegraphics[width=0.4\linewidth]{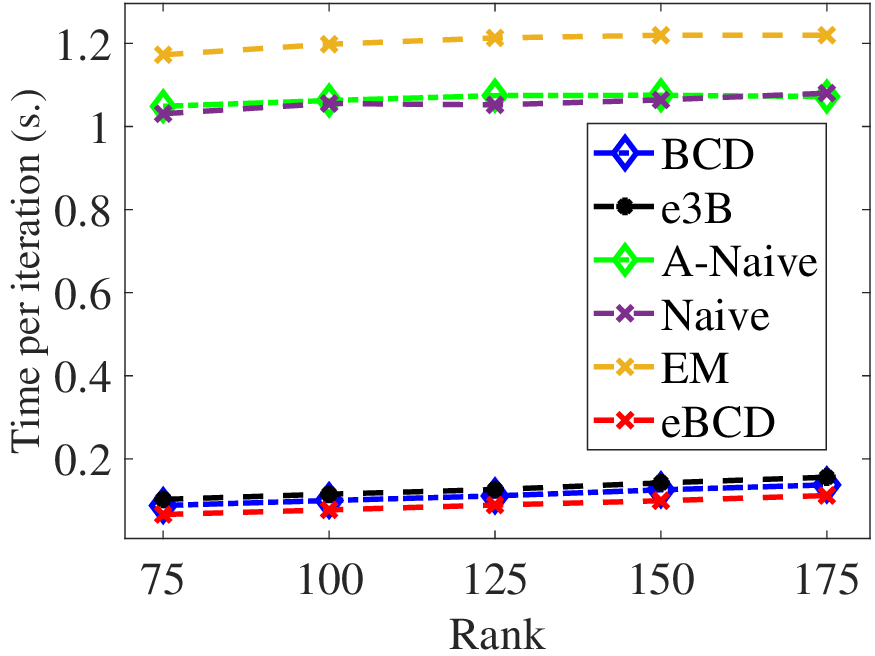}
\end{tabular}
\end{center}
\caption{MAD analysis and average iteration time for increasing values of the rank: The top two figures are for the \textit{\textit{k1b}} dataset, and the bottom two for the \textit{hitech} dataset.  \label{Fig:rank_an_text}} 
\end{figure}  
Even though no algorithm outperforms the others, Figure~\ref{Fig:rank_an_text} shows that using the \revise{3B-RMD} is to be preferred in this application. Indeed, \revise{\name}, \revise{\ename} and \revise{e3B} achieve the lowest MAD in both experiments. Moreover, these three algorithms have the lowest average iteration time, being more than  one second faster per iteration than \revise{Naive}, \revise{A-Naive}, and \revise{EM}. For this reason, \revise{\name}, \revise{\ename}, and \revise{e3B} perform more iterations than \revise{Naive}, \revise{A-Naive}, and \revise{EM} within 60 seconds, resulting in more accurate approximations.

\subsection{Compression of sparse data}

The last application is the compression of sparse data. We consider sparse greyscale images: 10000 randomly selected images from MNIST and Fashion MNIST (fMNIST), both of size $784 \times 10000$, and on two well-known sparse images: Phantom ($256 \times 256$)  and Satellite ($256 \times 256$). 
We also use sparse, structured matrices from various applications from~\cite{dai2012geometric}: Trec11 ($235 \times 1138$), mycielskian ($767 \times 767$), lock ($1074 \times 1074$), and beaconfd ($173 \times 295$). 
We choose the approximation rank $r$ to ensure a desired level of compression. Specifically, when computing the \revise{RMD} factors $W$ and $H$ of a sparse matrix $X \in \mathbb{R}^{m \times n}$ with $\nnz(X)$ non zero entries, the number of entries of $WH$ is $r(n+m)$. Hence, the rank $r$ is set to the closest integer such that $r\leq 0.5 \cdot \nnz(X) /  (n+m) $, which means a compression of $50\%$.  
 Table~\ref{tab:res_comp} shows the relative error 
$ \frac{\| X - \max(0,\Bar{W}\Bar{H}) \|_F}{\| X \|_F}$,
 where $\Bar{W}$ and $\Bar{H}$ are the computed factors. 
 \revise{
 Since we are considering real-world data of different dimensions, using a single time limit that is suitable for all the dataset is not meaningful. 
 Therefore, we assign a time limit to all the algorithms (except the TSVD) scaled by the dimensions of each data set: we set a maximum time limit of $T=300$ seconds for the largest data sets, namely MNIST and fMNIST with $m_{\max}=784$ and $n_{\max}=10000$. 
 Then, we assign a scaled time limit $T_{\text{d}}$ for a data set of dimension $m_{\text{d}} \times n_{\text{d}}$ using
 $T_{\text{d}} = T \frac{m_{\text{d}} \times n_{\text{d}}}{m_{\max} \times n_{\max}}$ seconds}, 
 and display the average relative error over 10 different random initializations.  
 
\begin{table}[h!]
\centering
\begin{tabular}{ c || c | c | c | c | c | c | c || c } 
 \toprule 
 \textbf{Data set}  & & \hspace{-0.1cm}\name\hspace{-0.3cm}&\hspace{-0.1cm}\ename\hspace{-0.3cm}&\hspace{-0.1cm}e3B\hspace{-0.3cm}&\hspace{-0.1cm}Naive\hspace{-0.3cm}& \hspace{-0.1cm}A-Naive\hspace{-0.3cm}&\hspace{-0.1cm}EM\hspace{-0.3cm}&\hspace{-0.15cm}TSVD\hspace{-0.2cm}  \\
  \midrule
    \hspace{-0.10cm} MNIST (300 s.) \multirow{2}{*}{}& err & 12.2\%   & 11.6\%   & \textbf{11.5\%} &   12.3\% &   11.6\%  &   12.9\% & 25.8\%  \\
 $r=70$ &  iter & 1215 &  2159  & 1075 & 1088  & 899  & 413 & /  \\
   \hline
 \hspace{-0.10cm} fMNIST (300 s.)  \multirow{2}{*}{}  & err  & 9.6\%& \textbf{9.1\%}& 9.2\%& 9.5\% & \textbf{9.1\%}& 9.7\%& 14.0\%\\
  $r=181$ &  iter & 703 &  1498  & 637 & 914  & 777  & 513 & /   \\
   \hline
\hspace{-0.10cm} Phantom (2.5 s.) \multirow{2}{*}{} & err & 9.0\%& \textbf{6.4\%} & 7.1 & 9.6\% & 7.8\%&9.8\% & 19.2\%\\
 $r=26$ &  iter & 540 &  2898  & 508 &  374 & 340  &  278 &  / \\
   \hline
 \hspace{-0.10cm} Satellite (2.5 s.) \multirow{2}{*}{} & err  & 16.1\% & \textbf{14.9\%} & 15.0\% & 17.0\% & 15.6\% & 18.2\% &26.0\% \\
 $r=12$  &  iter & 1052 & 4385   & 983 & 374  &  352 & 252 &  / \\
   \hline
  \hspace{-0.10cm} Trec11 (10.2 s.) \multirow{2}{*}{} & err & 31.3\% & 28.9\% & \textbf{28.8\%} &30.6\%& \textbf{28.8\%} & 35.7\%& 57.3\%\\
$r=13$   &  iter & 526 & 902   & 398  & 691  &461  &  228 & /  \\
   \hline
 \hspace{-0.10cm}  mycielskian (22.5 s.)\multirow{2}{*}{} & err & 3.6\% &  \textbf{0.6\%} & 0.9\% &8.0\%& 1.3\% & 9.0\%& 58.0\%\\
  $r=14$ &  iter & 516 &  1021  & 400 & 232  & 213  & 106 &  / \\
   \hline
  \hspace{-0.10cm} lock1074 (44.2 s.) \multirow{2}{*}{} &  err & 6.3\% & \textbf{0.1\%} & 0.2\% & 15.6\% & 1.9\% & 29.2\% & 67.3\% \\
 $r=12$  &  iter & 576 &  1158  & 432  & 238  & 214  & 96 &  / \\
   \hline
 \hspace{-0.10cm} beaconfd (2.0 s.) \multirow{2}{*}{} & err & 22.8 \% & 22.0\% & \textbf{21.7}\% &23.3\%& \textbf{21.7}\%& 37.3\%& 39.12\%\\
 $r=3$  &  iter & 998  &  1514  & 680 & 525  &  419 & 206 & /  \\
   \hline
 \bottomrule 
\end{tabular}
\caption{Compression of sparse matrices: \revise{we compare the average final relative error and the average iteration number, within the given time limit.} The best values are highlighted in bold.} 
\label{tab:res_comp}
\end{table}

We observe that \revise{RMD} outperform the TSVD, as all the algorithms yield approximations that are significantly more accurate than the TSVD. Additionally, \revise{\ename \   produces the most accurate solutions in five out of eight examples}; moreover, \revise{\ename} outperforms \revise{\name}, achieving a smaller compression error for all the data sets considered. We also observe that the structure of the matrix has a \revise{crucial} impact on the effectiveness of the compression using \revise{RMD}. 
\revise{For example, for lock1074, the improvement of RMD is impressive: the relative error of eBCD is 0.1\% while that of the TSVD is 67.3\%. On the other hand, the improvement for fMNIST is less pronounced: from 14\% relative error for the  TSVD to 9.6\% for RMD.}

Figure~\ref{Fig:vis_rep} visually illustrates the difference between TSVD and \revise{RMD} for  $50\%$ compression. 
One can visually appreciate that the \revise{\ename} approximation is consistently more accurate than the TSVD, as hinted by the errors  reported in Table~\ref{tab:res_comp}. 
\begin{figure}[h]
\begin{center}
    \begin{tabular}{ccc}
   \includegraphics[width=0.25\linewidth]{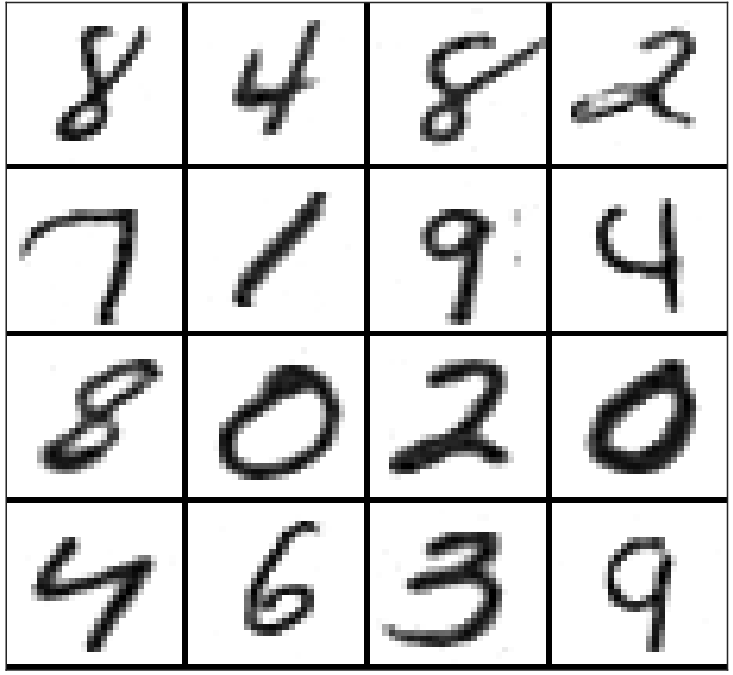}      &    \includegraphics[width=0.25\linewidth]{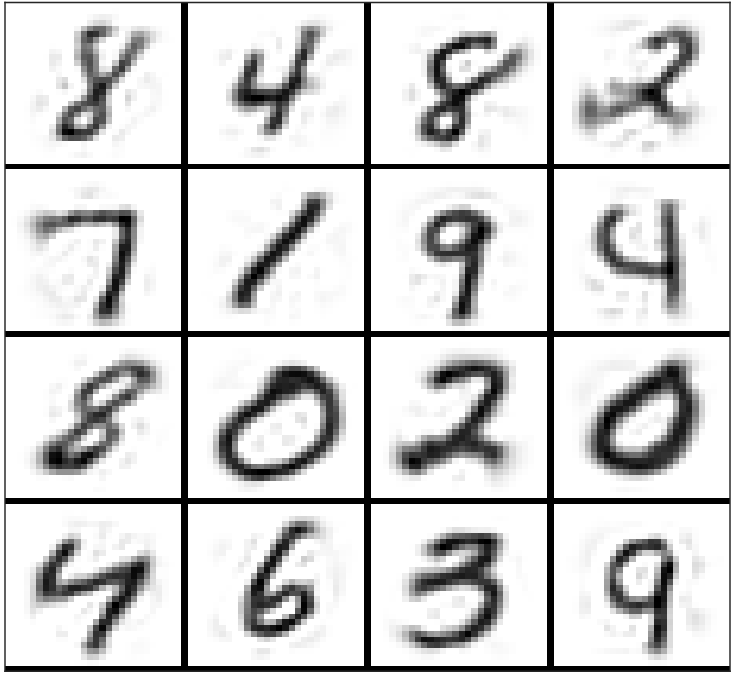}       
   &  \includegraphics[width=0.25\linewidth]{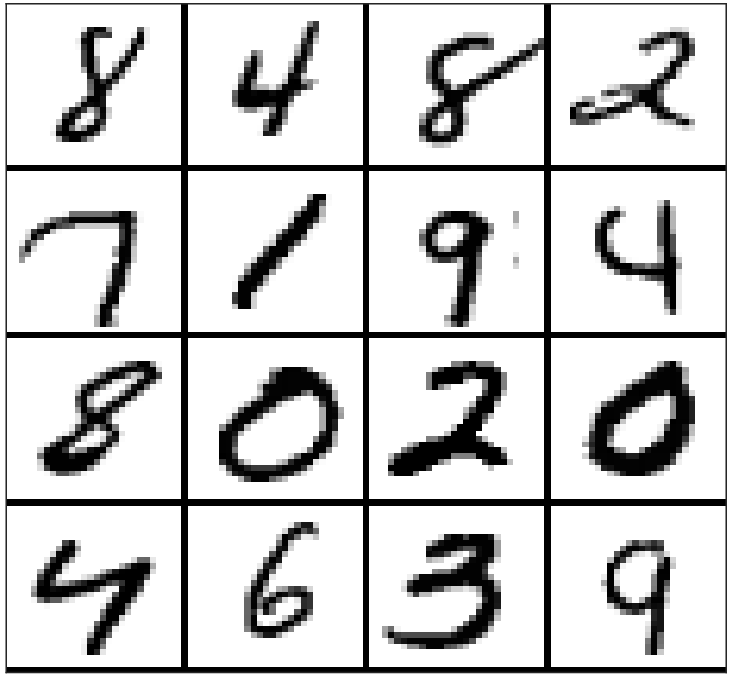} \\ 
       Original images & TSVD & \ename 
    \end{tabular} 
\end{center} 
\caption{Visual illustration: TSVD \revise{(relative error = 25.8\%)} against \revise{RMD with eBCD} \revise{(relative error = 11.6\%)} on the MNIST set with $50 \%$ compression ($r=70$).\label{Fig:vis_rep}} 
\end{figure}

\section{Conclusion}

The \revise{RMD} is a relatively recent matrix decomposition model (Saul, 2022~\cite{saul2022geometrical}) that finds applications in the compression of sparse data, entry-dependent matrix completion, and manifold learning. Moreover, we showed, for the first time, how it can be used for the completion of Euclidean distance matrices when only the smallest (or largest) entries are observed. 
Our first main contribution is to establish the non-equivalence of the two main models used to \revise{to compute RMDs}, namely \revise{LS-RMD} and \revise{Latent-RMD} (Corollary~\ref{cor:not_equivalence}). Furthermore, we derive an explicit connection between the optimal function values of the two, demonstrating that the solution of \revise{Latent-RMD} is, in the worst case, twice larger than that of \revise{LS-RMD} (Theorem~\ref{th:bound_for}). 
Our second main contribution is to prove the convergence of the \revise{\name} scheme for solving \revise{3B-RMD} by demonstrating that it falls into the general framework of exact BCD, where one of the subproblems is strictly quasi-convex (Theorem~\ref{th:conv_bcd}). 
Our third and most important contribution is a novel algorithm, \revise{\ename}, which uses extrapolation to accelerate \revise{\name}. \revise{Our new algorithm} extends the well-known LMaFit algorithm from~\cite{wen2012solving} for matrix completion. We proved the subsequence convergence of the eBCD-NMD scheme to a stationary point under mild assumptions (Theorem~\ref{th:conv_ebcd}). 
Finally, we provided extensive numerical results for 4 applications on synthetic and real data sets, showing that \revise{\ename} performs favorably compared to the to the state of the art, while having convergence guarantees which other algorithms lack. 

Future work includes investigating the well-posedness of the \revise{LS-RMD} in (\ref{eq:ReLU_NMD}), 
designing randomized algorithms to tackle large-scale problems, and  
exploring more advanced \revise{RMDs} such as 
$X\approx W_1H_1 + \max(0,W_2 H_2) - \max(0,W_3 H_3)$ that could  handle mixed-signed and non-sparse data. 


\revise{
\section*{Acknowledgments}  We are grateful to the anonymous reviewers who
carefully read the manuscript, their feedback helped us improve our paper.  
}

\section{Appendix} \, 

\subsection{Proof of Lemma~\ref{lem:sig_WH}}  
\begin{proof}
 Let $(W(\alpha),H(\alpha))$ be generated by (\ref{eq:eBCD-NMD-scheme_v2}), and let $S=\mathcal{P}_{\Omega}(X-WH)-\mathcal{P}_{\Omega^C}(\max(0,WH))$ be the residual with respect to the $k$-th iterate $(W,H)$. Recall that $Z_\alpha=WH+\alpha S$, $P=H^\top (HH^\top)^{-1} H$, and $E(\alpha)=W(\alpha)W(\alpha)^T$. We denote $Q(\alpha)R(\alpha) \Pi=Z_\alpha H^T$, the QR factorization of $Z_\alpha H^T$ with column pivoting, where $\Pi$ is a permutation matrix. Note that $\Pi$ is needed only if $Z_\alpha H^T$ is rank-deficient; otherwise, one can consider $\Pi$ as the identity matrix. We have 
 \begin{equation}
     W(\alpha) H(\alpha)-WH=\left[ W(\alpha)R(\alpha)\Pi (HH^T)^{-1}-W\right]H + W(\alpha)\left[H(\alpha)- R(\alpha)\Pi (HH^T)^{-1}H\right].
     \label{eq:diff_wh}
 \end{equation}
The first term in (\ref{eq:diff_wh}) can be written as 
 \begin{equation}
 \begin{aligned}
     \left[ W(\alpha)R(\alpha)\Pi (HH^T)^{-1}-W\right]H = &  W(\alpha)R(\alpha)\Pi (HH^T)^{-1}H-WH\\
     = & Z_\alpha H^T (HH^T)^{-1}H - W (HH^T) (HH^T)^{-1}H=Z_\alpha P - W H P\\
     = & (Z_\alpha-WH)P=\alpha S P.
     \end{aligned}
     \label{eq:first_term_wh_diff}
 \end{equation}
 Similarly, we write the second term of the summation in terms of the residual.  Using (\ref{eq:first_term_wh_diff}), 
 $$
 \begin{aligned}
    H(\alpha)=&W(\alpha)^T Z_\alpha=W(\alpha)^T(WH+\alpha S) \\
    =&W(\alpha)^T\left[W(\alpha)R(\alpha)\Pi (HH^T)^{-1}H-(W(\alpha)R(\alpha)\Pi (HH^T)^{-1}-W)H+\alpha S \right] \\
    =&W(\alpha)^T\left[W(\alpha)R(\alpha)\Pi (HH^T)^{-1}H+\alpha S (I-P) \right].
    \end{aligned}   
 $$
 We multiply both sides by $W(\alpha)$ to obtain 
 $$
 \begin{aligned}
    W(\alpha)H(\alpha)=&W(\alpha)W(\alpha)^T\left[W(\alpha)R(\alpha)\Pi (HH^T)^{-1}H+\alpha S (I-P) \right]\\
    =&W(\alpha)R(\alpha)\Pi (HH^T)^{-1}H+\alpha E(\alpha)S (I-P). 
    \end{aligned}
 $$
Rearranging the terms, we have 
 \begin{equation}
    W(\alpha)\left[H(\alpha)-R(\alpha)\Pi (HH^T)^{-1}H\right]=\alpha E(\alpha)S(I-P).
     \label{eq:second_term_wh_diff}
 \end{equation}
 Finally, combining (\ref{eq:diff_wh}), (\ref{eq:first_term_wh_diff}), and (\ref{eq:second_term_wh_diff}) we obtain
    \begin{equation}
        \begin{aligned}
            W(\alpha)H(\alpha)-WH=& \left[ W(\alpha)R(\alpha)\Pi (HH^T)^{-1}-W\right]H + W(\alpha)\left[H(\alpha)- R(\alpha) \Pi (HH^T)^{-1}H\right]\\
            =&\alpha S P +\alpha Q(\alpha) S (I-P)\\
            =& \alpha (I-Q(\alpha))S P + \alpha Q(\alpha) S.
            \label{eq:diff_WH}
        \end{aligned}
    \end{equation}
    From the properties of the orthogonal projection, (\ref{eq:diff_WH}) implies
    \begin{equation}
        \lVert  W(\alpha)H(\alpha)-WH \rVert_F^2=\alpha^2 \lVert (I-Q(\alpha))S P \rVert_F^2 +\alpha^2 \lVert Q(\alpha) S\rVert_F^2,
        \label{eq:norm_equival}
    \end{equation}
    which proves the second inequality in (\ref{eq:sigma_12}).
    
    Let us now obtain the first inequality. From (\ref{eq:diff_WH}) and the properties of orthogonal projection, 
    \begin{equation}
        \begin{aligned}
            \langle \alpha S, W(\alpha)H(\alpha)-WH \rangle=& \alpha^2 \langle  S, (I-E(\alpha))S P \rangle + \alpha^2 \langle S, E(\alpha) S \rangle \\
            =&\alpha^2 \lVert (I-E(\alpha))S P \rVert_F^2 + \alpha^2  \lVert E(\alpha) S\rVert_F^2\\
            =& \lVert  W(\alpha)H(\alpha)-WH \rVert_F^2,
        \end{aligned}
        \label{eq:scal_prod_dif_WH}
    \end{equation} 
From the observation (\ref{eq:res_corr_Z}) on $Z_\alpha$,  the definition of $\sigma_{WH}(\alpha)$, and from (\ref{eq:scal_prod_dif_WH}), we get
    $$
        \begin{aligned}
            \sigma_{WH}(\alpha)=& \lVert S \rVert_F^2 -\frac{1}{\alpha^2} \lVert  W(\alpha)H(\alpha)-Z_\alpha \rVert_F^2=\lVert S \rVert_F^2 -\frac{1}{\alpha^2} \lVert  W(\alpha)H(\alpha)-WH-\alpha S \rVert_F^2 \\
           = &\lVert S \rVert_F^2 -\frac{1}{\alpha^2} \left[ \lVert  W(\alpha)H(\alpha)-WH \rVert_F^2+\alpha^2 \lVert S \rVert_F^2- 2 \alpha \langle S, W(\alpha)H(\alpha)-WH \rangle \right]\\
            =&\frac{1}{\alpha^2} \lVert  W(\alpha)H(\alpha)-WH \rVert_F^2,
        \end{aligned}
    $$
    which proves the first equality in (\ref{eq:sigma_12}).
\end{proof}

\subsection{Proof of Lemma~\ref{lem:sigm_pc_z}}

\begin{proof}
Let us first show the continuity of the function in (\ref{eq:sigma_z_pc}) when $\alpha \to 1^+$. Recall that $Z_1=Z$ and that $W(\alpha)$ is an orthogonal basis of $\mathcal{R}(Z_\alpha H^T)$. Therefore, the orthogonal projection operator into $\mathcal{R}(Z_\alpha H^T)$ can be written as $W(\alpha)W(\alpha)^T=(Z_\alpha H^T)(Z_\alpha H^T)^\dagger$. If rank($Z_{\alpha}H^T$)=rank($ZH^T$) in some interval $[1,\widehat{\alpha}]$, then by the continuity of the Moore-Penrose inverse~\cite{stewart1969continuity}, we have \begin{equation}
     \lim_{\alpha \to 1^+}W(\alpha) W(\alpha)^T = \lim_{\alpha \to 1^+}(Z_\alpha H^T)(Z_\alpha H^T)^\dagger=(Z H^T)(Z H^T)^\dagger=W(1) W(1)^T.
    \label{eq:lim_WH_alpha}
\end{equation}
From (\ref{eq:lim_WH_alpha}) and the continuity of the product, 
\begin{equation}
\begin{aligned}
     \lim_{\alpha \to 1^+}W(\alpha) H(\alpha)= &\lim_{\alpha \to 1^+}W(\alpha) W(\alpha)^T Z_\alpha = W(1) W(1)^T Z=W(1) H(1),\\
     \end{aligned}
    \label{eq:lim_WH_1} 
\end{equation}
which leads to the continuity of (\ref{eq:sigma_z_pc}) when  $\alpha \to 1^+$. 
Let us recall that 
\[
Z_{\alpha}=WH+\alpha S, \  P_{\Omega^C}(S)=-P_{\Omega^C}(\max(0,WH)), \  P_{\Omega^C}(Z(\alpha))=P_{\Omega^C}(\min(0,(W(\alpha)H(\alpha))).  
\] 
 We have 
    \begin{equation}
        \begin{aligned}   
    &\frac{1}{\alpha^2} \lVert P_{\Omega^C}(W(\alpha)H(\alpha)-Z_{\alpha}) \rVert_F^2-\lVert P_{\Omega^C}(S(\alpha)) \rVert_F^2\\
    =&\frac{1}{\alpha^2} \lVert P_{\Omega^C}(W(\alpha)H(\alpha)-Z_\alpha )\rVert_F^2-\lVert P_{\Omega^C}(W(\alpha)H(\alpha)-Z(\alpha)) \rVert_F^2\\
    =&\frac{1}{\alpha^2} \lVert P_{\Omega^C}(W(\alpha)H(\alpha)-WH- \alpha S) \rVert_F^2-\lVert P_{\Omega^C}(W(\alpha)H(\alpha)-\min(0,(W(\alpha)H(\alpha))) \rVert_F^2\\
    =& \frac{1}{\alpha^2} \lVert P_{\Omega^C}(W(\alpha)H(\alpha)-\min(0,WH)+(\alpha-1)\max(0,WH)) \rVert_F^2-\\
    &-\lVert P_{\Omega^C}(W(\alpha)H(\alpha)-\min(0,(W(\alpha)H(\alpha))) \rVert_F^2.
    \end{aligned}
    \label{eq:rel_sigma_z_1}
    \end{equation} 
Since (\ref{eq:sigma_z_pc}) is continuous for $\alpha \to 1^+$, we just need to show that it is larger than or equal to zero  at $\alpha=1$. Using (\ref{eq:lim_WH_1}) and (\ref{eq:rel_sigma_z_1}), 
$$
    \begin{aligned}
    & \lVert P_{\Omega^C}(W(1)H(1)-\min(0,WH)) \rVert_F^2-\lVert P_{\Omega^C}(\max(0,(W(1)H(1))) \rVert_F^2\\
    =&\lVert P_{\Omega^C}(W(1)H(1))  \rVert_F^2  + \lVert P_{\Omega^C}(\min(0,WH)) \rVert_F^2-2 \langle P_{\Omega^C}(W(1)H(1)), P_{\Omega^C}(\min(0,WH)) \rangle \\
    &-\lVert P_{\Omega^C}(\max(0,(W(1)H(1))) \rVert_F^2\\
    =&\lVert P_{\Omega^C}(\min(0,W(1)H(1)))  \rVert_F^2  + \lVert P_{\Omega^C}(\min(0,WH)) \rVert_F^2-\\
    &-2 \langle P_{\Omega^C}(\min(0,W(1)H(1))), P_{\Omega^C}(\min(0,WH)) \rangle \\ 
    & -2 \langle P_{\Omega^C}(\max(0,W(1)H(1))), P_{\Omega^C}(\min(0,WH)) \rangle\\
    =&\lVert P_{\Omega^C}(\min(0,W(1)H(1))-(\min(0,WH))  \rVert_F^2 \\ 
    & -2 \langle P_{\Omega^C}(\max(0,W(1)H(1))), P_{\Omega^C}(\min(0,WH)) \rangle \geq 0,
    \end{aligned} 
$$ 
which concludes the proof. 
\end{proof}

\subsection{Proof of Lemma~\ref{lem:sig_Z}}

\begin{proof}   
    By an analogous argument to the one used in Lemma~\ref{lem:sig_Z}, we get the continuity of (\ref{eq:sigma_z}) as $\alpha \to 1^+$. Next, we show that (\ref{eq:sigma_z}) is zero at $\alpha=1$. By the definition of the residual in (\ref{eq:residual}), 
    $$ 
    \begin{aligned}
       &\frac{1}{\alpha^2}\lVert P_{\Omega}(W(\alpha)H(\alpha)-Z_{\alpha}) \rVert_F^2-\lVert P_{\Omega}(S(\alpha)) \rVert_F^2\\
       = & \lVert P_{\Omega}( W(\alpha)H(\alpha) -X +X-W H-\alpha S) \rVert_F^2-\lVert P_{\Omega}(S(\alpha)) \rVert_F^2\\
       = & \lVert -P_{\Omega}(S(\alpha))+ (1-\alpha)P_{\Omega}(S)) \rVert_F^2-\lVert P_{\Omega}(S(\alpha)) \rVert_F^2,
        \end{aligned}
   $$ 
   which is equal to zero for $\alpha=1$.
\end{proof}

\bibliographystyle{spmpsci}
\bibliography{bib}

\end{document}